\documentclass[twoside,10pt]{article}
\usepackage{jmlr2e}
\usepackage{mathbbol}
\usepackage{epsfig}
\usepackage{algorithm,algorithmic}

\newcommand{\tr}{\mathrm{tr}}
\newcommand{\pian}[2]{\frac{\partial #1}{\partial #2}}

\firstpageno{1}
\begin{document}

\title{Towards Optimal One Pass Large Scale Learning with Averaged Stochastic Gradient Descent}
\author{\name Wei Xu \email emailweixu@fb.com \\
\addr Facebook, Inc. $^1$ \\
1601 S. California Ave \\
Palo Alto, CA 94304, USA }

\footnotetext[1]{Major part of the work was done when the author was at NEC Labs America, Inc.}

\editor{}

\date{}
\maketitle
\begin{abstract}
For large scale learning problems, it is desirable if we can
obtain the optimal model parameters by going through the data in
only one pass. \cite{Polyak92} showed that asymptotically the test
performance of the simple average of the parameters obtained by
stochastic gradient descent (SGD) is as good as that of the
parameters which minimize the empirical cost. However, to our
knowledge, despite its optimal asymptotic convergence rate,
averaged SGD (ASGD) received little attention in recent research
on large scale learning. One possible reason is that it may take a
prohibitively large number of training samples for ASGD to reach
its asymptotic region for most real problems. In this paper, we
present a finite sample analysis for the method of
\cite{Polyak92}. Our analysis shows that it indeed usually takes a
huge number of samples for ASGD to reach its asymptotic region for
improperly chosen learning rate. More importantly, based on our
analysis, we propose a simple way to properly set learning rate so
that it takes a reasonable amount of data for ASGD to reach its
asymptotic region. We compare ASGD using our proposed learning
rate with other well known algorithms for training large scale
linear classifiers. The experiments clearly show the superiority
of ASGD.
\end{abstract}

\section{Introduction}
For prediction problems, we want to find a function $f_\theta(x)$
with parameter $\theta$ to predict the value of the outcome
variable $y$ given an observed vector $x$. Typically, the problem
is formulated as an optimization problem:
\begin{eqnarray}
\theta_t^* = \arg \min_\theta \frac{1}{t} \sum_{i=1}^t
(L(f_\theta(x_i),y_i) + R(\theta)) \label{EqEmpiricalMin}
\end{eqnarray}
where $t$ is the number of data points, $\theta_t^*$ is the
parameter that minimize the empirical cost, $(x_i,y_i)$ are the
$i^{th}$ training example, $L(s,y)$ is a loss function which gives
small value if $s$ is a good prediction for $y$, and $R(\theta)$
is a regularization function for $\theta$ which typically gives
small value for small $\theta$. Some commonly used $L$ are:
$\max(0,1-ys)$ for support vector machine (SVM),
$\frac{1}{2}(\max(0,1-ys))^2$ for L2 SVM, and $\frac{1}{2}(y-s)^2$
for linear regression. Some commonly used regularization functions
are: L2 regularization $\frac{\lambda}{2}\|\theta\|^2$, and L1
regularization $\lambda \|\theta\|_1$.

For large scale machine learning problems, we need to deal with
optimization problems with millions or even billions of training
samples. The classical optimization techniques such as interior
point methods or conjugate gradient descent have to go through all
 data points to just evaluate the objective once. Not to say
that they need to go through the whole data set many times in
order to find the best $\theta$.

On the other hand, stochastic gradient descent (SGD) has been
shown to have great promise for large scale learning
\citep{Zhang04,Hazan06,Shalev-Shwartz07,Bottou08,Shalev-Shwartz09,Langford09}.
Let $d=(x,y)$ be one data sample,
$l(\theta,d)=L(f_\theta(x),y)+R(\theta)$ be the cost of $\theta$
for $d$, $g(\theta,\xi)=\pian{l(\theta,d)}{\theta}$ be the
gradient function, and $D_t=(d_1,\cdots,d_t)$ be all the training
samples at $t^{th}$ step. The SGD method updates $\theta$
according to its stochastic gradient:
\begin{eqnarray}
\theta_t = \theta_{t-1} - \gamma_t g(\theta_{t-1},d_t)
\label{EqSGD}
\end{eqnarray}
where $\gamma_t$ is learning rate at the $t^{th}$ step. $\gamma_t$
can be either a scalar or a matrix. Let the expected loss of
$\theta$ over test data be $\mathcal{E}(\theta)=E_d(l(\theta,d))$,
the optimal parameter be $\theta^*=\arg\min_\theta
\mathcal{E}(\theta)$, and the Hessian be $H=\left.
\pian{^2\mathcal{E}(\theta)}{\theta\partial
\theta^T}\right|_{\theta=\theta^*}$. Note that $\theta_t$ and
$\theta_t^*$ are random variables depending on $D_t$. Hence both
$\mathcal{E}(\theta_t)$ and $\mathcal{E}(\theta_t^*)$ are random
variables depending on $D_t$. If $\gamma_t$ is a scalar, the best
asymptotic convergence for the expected excess loss
$E_{D_t}(\mathcal{E}(\theta_t))-\mathcal{E}(\theta^*)$ is
$O(t^{-1})$, which is obtained by using
$\gamma_t=\gamma_0(1+\gamma_0\lambda_0t)^{-1}$, where $\lambda_0$
is the smallest eigenvalue of $H$ and $\gamma_0$ is some constant.
The asymptotic convergence rate of SGD can be potentially benefit
from using second order information
\citep{Bottou08,Schraudolph07,Amari00}. The optimal asymptotic
convergence rate is achieved by using matrix valued learning rate
$\gamma_t = \frac{1}{t}H^{-1}$. If this optimal matrix step size
is used, then asymptotically second order SGD is as good as
explicitly optimizing the empirical loss. More precisely, this
means that both
$tE_{D_t}(\mathcal{E}(\theta_t)-\mathcal{E}(\theta^*))$ and
$tE_{D_t}(\mathcal{E}(\theta_t^*)-\mathcal{E}(\theta^*))$ converge
to a same positive constant.

Since $H$ is unknown in advance, methods for adaptively estimating
$H$ is proposed \citep{Bottou05,Amari00}. However, for high
dimensional data sets, maintaining a full matrix $H$ is too
computationally expensive. Hence various methods for approximating
$H$ have been proposed
\citep{LeCun98,Schraudolph07,Roux08,Bordes09}. However, with the
approximated $H$, the optimal convergence cannot be guaranteed. It
is worth to point out that most of the existing analysis for
second order SGD is asymptotic, namely, that they do not tell how
much data is needed for the algorithm to reach their asymptotic
region.

In order to accelerate the convergence speed of SGD, averaged
stochastic gradient (ASGD) was proposed in \cite{Polyak92}. For
ASGD, the running average $\bar{\theta}_t=\frac{1}{t}\sum_{j=1}^t
\theta_j$ of the parameters obtained by SGD is used as the
estimator for $\theta^*$. \cite{Polyak92} showed a very nice
result that $\bar{\theta}_t$ converges to $\theta^*$ as good as
full second order SGD, which means that if there are enough
training samples, ASGD can obtain the parameter as good as the
empirical optimal parameter $\theta_t^*$ in just one pass of data.
And another advantage of ASGD is that, unlike second order SGD,
ASGD is extremely easy to implement. \cite{Zhang04,Nemirovski09}
gave some nice non-asymptotic analysis for ASGD with a fixed learning
rate. However, the convergence bounds obtained by
\cite{Zhang04,Nemirovski09} are far less appealing than that of
\cite{Polyak92}.

Despite its nice properties, ASGD receives little attention in
recent research for online large scale learning. The reason for
the lack of interest in ASGD might be that its potential good
convergence has not been realized by researchers in real
applications. Our analysis shows the cause of this may due to the
fact the ASGD needs a prohibitively large amount of data to reach
asymptotics if learning rate is chosen arbitrarily.

A typical choice for the learning rate $\gamma_t$ is to make it
decease as fast as $\Theta(t^{-c})$ for some constant $c$. In this
paper, we assume a particular form of learning rate schedule which
satisfies this condition,
 \begin{equation} \label{EqGammaSchedule} \gamma_t=\gamma_0(1+a\gamma_0 t)^{-c} \end{equation}
where $\gamma_0$, $a$ and $c$ are some constants. Based on this
form of learning rate schedule, we provide non-asymptotic analysis
of ASGD. Our analysis shows that $\gamma_0$ and $a$ should to be
properly set according to the curvature of the expected cost
function. $c$ should be a problem independent constant. With our
recipe for setting the learning rate, we show that ASGD
outperforms SGD if the data size is large enough for SGD to reach
its asymptotic region.


To demonstrate the effectiveness of ASGD with the proposed
learning rate schedule, we apply ASGD for training linear
classification and regression models. We compare ASGD with other
prominent large scale SVM solvers on several benchmark tasks. Our
experimental results show the clear advantage of ASGD.

In the rest of the paper, for matrices $X$ and $Y$, $X\le Y$ means
$Y-X$ is positive semi-definite, $\|x\|_A$ is defined as
$\sqrt{x^T A x}$. We will assume $\gamma_t=\gamma_0(1+a\gamma_0
t)^{-c}$ for some constant $\gamma_0>0$, $a>0$ and $0\le c \le 1$
in all the theorems and lemmas. Through out this paper we denote
$\Delta_t=\theta_t-\theta^*$ and
$\bar{\Delta}_t=\bar{\theta}_t-\theta^*$. To help the reader focus
on the main idea, we put most proofs to the Appendix.

The paper is organized as follows: Section \ref{SecLinear}
establish some results on stochastic linear equation; Section
\ref{SecRegression} extends the result to ASGD for quadratic loss
functions; Section \ref{SecNonquadratic} works on general
non-quadratic loss functions; Section \ref{SecImplementation}
discusses some implementation issues; Section \ref{SecExperiments}
shows experimental results; Section \ref{SecConclusion} concludes
the paper; and Appendix includes all the proofs.

\section{Stochastic Linear Equation}
\label{SecLinear}

To motivate the problem, we first take a close look at the SGD
update (\ref{EqSGD}). Let $\bar{g}(\theta)=E(g(\theta,d))$ and the
first order Taylor expansion of $\bar{g}(\theta)$ around
$\theta^*$ be $A\theta-b$, where $A=\left.
\pian{\bar{g}(\theta)}{\theta}\right|_{\theta=\theta^*}$ and
$b=A\theta^*-\bar{g}(\theta^*)=A\theta^*$. Then
$g(\theta_{t-1},d)$ can be decomposed as:
\begin{eqnarray*}
g(\theta_{t-1},d)&=& (A\theta_{t-1}-b) + g(\theta^*,d) + (g(\theta_{t-1},d)-g(\theta^*,d)-\bar{g}(\theta_{t-1})) + (\bar{g}(\theta_{t-1})-A\theta_{t-1}+b)  \\
&=& (A\theta_{t-1}-b) + \xi_t^{(1)} + \xi_t^{(2)}  + \xi_t^{(3)}
\end{eqnarray*}
where $\xi_t^{(1)}=g(\theta^*,d_t)$,
$\xi_t^{(2)}=g(\theta_{t-1},d_t)-g(\theta^*,d_t)-\bar{g}(\theta_{t-1})$
and $\xi_t^{(3)}=\bar{g}(\theta_{t-1})-A\theta_{t-1}+b$. So the
SGD update (\ref{EqSGD}) can be re-written as
\begin{eqnarray}
\theta_t=\theta_{t-1}-\gamma_t (A\theta_{t-1}-b + \xi_t^{(1)} + \xi_t^{(2)}  +
\xi_t^{(3)}) \label{EqSGDDecomposed}
\end{eqnarray}
It is easy to see that $\xi_t^{(1)}$ is martingale with respect to
$d_t$, i.e., $E(\xi_t^{(1)}|d_1,\cdots,d_{t-1})=0$, and has
identical distribution for different $t$. $\xi_t^{(2)}$ is also
martingale with respect to $d_t$. However, as we will see in later
section, its magnitude depends on $\theta_{t-1}-\theta^*$. If
$g(\theta,d)$ is smooth, we have
$\xi_t^{(2)}=O(\|\theta_{t-1}-\theta^*\|)$. For smooth
$\bar{g}(\theta)$, we have
$\xi_t^{(3)}=o(\|\theta_{t-1}-\theta^*\|)$. Both $\xi_t^{(2)}$ and
$\xi_t^{(3)}$ are asymptotically negligible if suitable conditions
are met. We also note that $\xi_t^{(3)}=0$ for quadratic
$l(\theta,\xi)$.

By the above analysis, we first consider the following simple
stochastic approximation procedure which ignores $\xi_t^{(2)}$ and
$\xi_t^{(3)}$:
\begin{eqnarray}
&& \theta_t=\theta_{t-1}-\gamma_t (A\theta_{t-1}-b+\xi_t) \label{EqSGDLinear}\\
&& \bar{\theta}_t=\frac{1}{t}\sum_{i=1}^t \theta_i \label{EqLinearAvg}
\end{eqnarray}
where $A$ is a positive definite matrix with the smallest
eigenvalue $\lambda_0$ and the largest eigenvalue $\lambda_1$,
$\xi_t$ is martingale difference process, i.e.,
$E(\xi_t|\xi_1,\cdots,\xi_{t-1})=0$, the variance of $\xi_t$ is
$E(\xi_t\xi_t^T) = S$. We will see that this algorithm can be used
to find the root $\theta^*$ of equation $A\theta=b$
\begin{theorem}
\label{ThmLinearAvg} If $\gamma_0 \lambda_1 \le 1$ and
$(2c-1)a<\lambda_0$, then the estimator $\bar{\theta}_t$ in
(\ref{EqLinearAvg}) satisfies:
\begin{eqnarray*}
t E(\|\bar{\theta}_t-\theta^*\|_A^2) &\le& \tr(A^{-1} S) +
\frac{(2c_0+c_0^2)(1+a \gamma_0 t)^{c-1}}{c} \tr(A^{-1} S) +
\frac{(1 + c_0)^2 }{\gamma_0^2 t}\|\theta_0-\theta^*\|_{A^{-1}}^2
\end{eqnarray*}
where
\[ c_0=\frac{ac(1+ac\gamma_0)}{(\lambda_0-\max(0,2c-1)a)}
\]
\end{theorem}
The immediate conclusion from Theorem \ref{ThmLinearAvg} is the
asymptotic convergence bound of $\bar{\theta}_t$.
\begin{corollary} \label{CoralloryLinearAymptoticConvergence} $\bar{\theta}_t$ in (\ref{EqLinearAvg}) satisfies
\[ t E(\|\bar{\theta}_t-\theta^*\|_A^2) \le \tr(A^{-1}S)+O(t^{-(1-c)}) \]
\end{corollary}
The above bound is consistent with Theorem 1 in \cite{Polyak92}
and is the best possible asymptotic convergence rate that can be
achieved by any algorithms \citep{Fabian73}. However, we are more
interested in the non-asymptotic behavior of $\bar{\theta}_t$.
\begin{corollary}
If we choose $a=\lambda_0$, it takes $t=O((\lambda_0
\gamma_0)^{-1})$ samples for $\bar{\theta}_t$ in
(\ref{EqLinearAvg}) to reach the asymptotic region. And at this
point, $\bar{\theta}_t$ begins to become better than $\theta_t$.
\end{corollary}
\begin{proof}
Let $t=\frac{K}{\lambda_0 \gamma_0}$, we have
\begin{equation}
E(\|\bar{\Delta}_t\|_A^2) \le \frac{(1 + c_0)^2 }{K^2}
\|\Delta_0\|_A^2 + \frac{\lambda_0 \gamma_0}{K}
\left(1+\frac{(2c_0+c_0^2)(1+K)^{c-1}}{c} \right)  \tr(A^{-1} S)
\label{EqAvgRate}
\end{equation}
On the other hand, the best possible convergence for $\theta_t$ is
obtained with $a=\lambda_0$ and $c=1$:
\begin{equation}
E\left(\|\Delta_t\|_A^2\right) \le
\frac{\|\Delta_0\|_A^2}{(1+K)^2} +\frac{\gamma_0 \tr(S)}{1+K}
\label{EqRate}
\end{equation}
We omit the proof of (\ref{EqRate}), which is similar to that of
Theorem \ref{ThmLinearAvg}. A related (but not exactly same) result
can be found in section 2.1 of \cite{Nemirovski09}. From (\ref{EqAvgRate}) and
(\ref{EqRate}) we can see that both $\theta_t$ and
$\bar{\theta}_t$ need $t=O((\lambda_0 \gamma_0)^{-1})$ to reach
their asymptotic region. However, at this point, $\bar{\theta}_t$
begins to become better than $\theta_t$ because $\lambda_0
\tr(A^{-1}S) \le \tr(S)$.
\end{proof}

\begin{corollary} \label{CorollarySampleSize2}
It takes
$t=\Omega\left(\left(\frac{a}{\lambda_0}\right)^{\frac{c}{1-c}}(\lambda_0
\gamma_0)^{-1}\right)$ samples for $\bar{\theta}_t$ in
(\ref{EqLinearAvg}) to reach the asymptotic region.
\end{corollary}
\begin{proof}
In order for $\bar{\theta}_t$ to reach its asymptotic region, we
need at least the second term of the right hand side of the bound
in Theorem \ref{ThmLinearAvg} to be less than $\tr(A^{-1}S)$, which
is to say
\[ 2\frac{c_0 (1+a\gamma_0t)^{c-1} }{c} \le 1 \]
Hence
\[ t \ge \frac{1}{a\gamma_0} \left(\frac{2c_0}{c}\right)^{\frac{1}{1-c}} = \left(\frac{2a}{\lambda_0}\right)^{\frac{c}{1-c}}(\lambda_0
\gamma_0)^{-1} \]
\end{proof}

By Corollary \ref{CorollarySampleSize2}, we should limit $a$ in
order to have fast convergence. For the linear problem
(\ref{EqSGDLinear}), we should always use $a=0$. If we use some
arbitrary value such as 1 for $a$, although $\bar{\theta}_t$ still
has asymptotic optimal convergence according to \cite{Polyak92},
but it needs much more samples to reach the asymptotic region in
situations where $\lambda_0$ is very small. For the general SGD
update (\ref{EqSGDDecomposed}), we need to trade-off against the
convergence of $\xi^{(2)}$ and  $\xi^{(3)}$. Hence $a$ should not
be 0. In general, $a$ should be a constant factor times of
$\lambda_0$.

\section{Regression Problem}
\label{SecRegression} In this section, we will analyze the
convergence for regression problems. As we noted in section
\ref{SecLinear}, the SGD update can be decomposed as
(\ref{EqSGDDecomposed}), where $\xi_t^{(3)}=0$ for quadratic loss
of linear regression. As in the proof of Theorem \ref{ThmLinearAvg},
$\bar{\Delta}_t$ can be written as:
\begin{eqnarray*}
 \bar{\Delta}_t = \frac{1}{\gamma_0 t} \bar{X}_0^t \Delta_0 + \frac{1}{t} \sum_{j=1}^t \bar{X}_j^t \xi_j^{(1)} + \frac{1}{t} \sum_{j=1}^t \bar{X}_j^t \xi_j^{(2)}  = I^{(0)} + I^{(1)} + I^{(2)}
\end{eqnarray*}
We already have a bound for $\|I^{(0)}\|_A$ and $\|I^{(1)}\|_A$ in Theorem \ref{ThmLinearAvg}. Now we work on $I^{(2)}$. We will make two assumptions:
\begin{eqnarray}
&& E\left(\left.\| \xi_j^{(2)}\|_{A^{-1}}^2\right|\theta_{j-1}\right)  \le c_1 \|\Delta_{j-1}\|_A^2  \label{A_Xi2Bound} \\
&& \sum_{i=j}^t E\left(\left.\|\Delta_t\|_A^2  \right|\theta_{j-1}
\right) \le c_2 \|\Delta_{j-1}\|_A^2 + c_3 \sum_{i=j}^t  \gamma_t
\label{A_DeltaBound}
\end{eqnarray}
(\ref{A_Xi2Bound}) is related to the continuity of $g(\theta,d)$
and the distribution of $y$. (\ref{A_DeltaBound}) is related to
the convergence of standard SGD. A bound similar to
(\ref{A_DeltaBound}) can be found in section 3.1 of \cite{Hazan06}. Using these
assumptions, we can bound $E\|I^{(2)}\|_A^2$:
\begin{lemma} \label{LemmaRegression}
With Assumption (\ref{A_Xi2Bound}) (\ref{A_DeltaBound}) , we have
\begin{eqnarray}
 t E\|I^{(2)}\|_A^2 \le (1+c_0)^2 c_1 \left(\frac{1+c_2}{t} \|\Delta_0\|_A^2 + \frac{c_3 \gamma_0}{1-c}(1+a\gamma_0t)^{-c}
 \right) \label{EqI2Bound}
\end{eqnarray}
\end{lemma}

With the above lemma, we can obtain the following asymptotic
convergence result:
\begin{corollary} For quadratic loss, with assumption (\ref{A_Xi2Bound}) (\ref{A_DeltaBound}), $\bar{\theta}_t$ satisfies
\[ tE\|\bar{\theta}_t-\theta^*\|_A^2 \le \tr(A^{-1}S) + O\left(t^{-c/2}\right)+O\left(t^{-(1-c)}\right) \]
\end{corollary}
\begin{proof}
Note that
\[ (E\|\bar{\Delta}_t\|_A^2)^{1/2} \le
(E\|I^{(0)}\|_A^2)^{1/2}+(E\|I^{(1)}\|_A^2)^{1/2}+(E\|I^{(2)}\|_A^2)^{1/2}
\]
The corollary follows by applying (\ref{EqI0Bound}),
(\ref{EqI1Bound}) and Lemma \ref{LemmaRegression}.
\end{proof}
The best convergence rate is obtained when $c=2/3$. Now we take a close look at the constant factor $c_1$ in
assumption (\ref{A_Xi2Bound}) to have a better understanding of
the non-asymptotic behavior of $t E\|I^{(2)}\|_A^2$.
\begin{lemma}\label{LemmaC1}
For ridge regression $l(\theta,d)=\frac{1}{2}(\theta^T x - y)^2$,
if $\|x\|\le M$, then
\[ E\left(\left.\| \xi_j^{(2)}\|_{A^{-1}}^2\right|\theta_{j-1}\right)  \le \frac{M}{\lambda_0} \|\Delta_{j-1}\|_A^2 \]
\end{lemma}

Assuming $\|x\|=M$, Lemma \ref{LemmaMaxGamma} in the Appendix
shows that $\|\Delta_t\|^2$ will diverge if learning rate is
greater than $\frac{2}{M}$. So $\gamma_0\le \frac{2}{M}$ and
$c_1\le \frac{M}{\lambda_0}$. Plugging these bounds for $c_1$ and
$\gamma_0$ into Lemma \ref{LemmaRegression}, we have the following
for $t=\frac{K}{\lambda_0\gamma_0}$,
\[ E\|I^{(2)}\|_A^2 \le 2(1+c_0)^2 \left( \frac{(1+c_2) \lambda_0 \gamma_0 \|\Delta_0\|_A^2}{K^2}+\frac{c_3  \gamma_0 }{(1-c)K(1+K)^c} \right) \]
Note that the best possible SGD error bound is
$\frac{\|\Delta_0\|_A^2}{(1+K)^2}+\frac{c_3 \gamma_0}{1+K}$ with
$a=\lambda_0$ and $c=1$. We see that $E\|I^{(2)}\|_A^2$ is negligible compared to
the error of SGD if $t > O((\lambda_0 \gamma_0)^{-1})$. Together
with the analysis in Section \ref{SecLinear}, we conclude that
ASGD begins to outperform SGD after $t>O((\lambda_0
\gamma_0)^{-1})$. The conclusion we draw in this section applies
not only to the case of $y$ with constant norm. Similar conclusion
can be drawn if $y$ is normally distributed or if each dimension
of $y$ is independently distributed, and/or if L2 regularization is used.

Based on above
analysis, for linear regression problems, we propose to use the
following values for (\ref{EqGammaSchedule}) to calculate the
learning rate: $ \gamma_0=1/M$, $a=\lambda_0$, $c=2/3$. We will
see that in the next section for general non-quadratic loss,
optimal $c$ is different since we need to further consider the
convergence of $\xi_t^{(3)}$.

\section{Non-quadratic loss} \label{SecNonquadratic}
For non-quadratic loss, we need to analyze the contribution of
$\xi^{(3)}$ to the error. We need the following two additional
assumptions:
\begin{eqnarray}
&& E\left(\left.\| \xi_j^{(3)} \|_{A^{-1}}\right|\theta_{j-1}\right) \le c_4 \|\theta_{j-1}-\theta^*\|_A^2 \label{A_Xi3Bound}\\
&& \sum_{i=1}^t E(\|\Delta_t\|_A^4 ) \le c_5 \|\Delta_0\|_A^4 +
c_6 \sum_{i=1}^t  \gamma_t \label{A_DeltaBound2}
\end{eqnarray}
Similar to (\ref{A_Xi2Bound}), (\ref{A_Xi3Bound}) is related to
the continuity of $g(\theta,d)$ and the distribution of $x$ and
$y$. Similar to (\ref{A_DeltaBound}), (\ref{A_DeltaBound2}) is
related to the convergence of standard SGD. We note that the
asymptotic normality of $\theta_t$ \citep{Fabian68} suggests that
assumption (\ref{A_DeltaBound2}) is reasonable.

\begin{lemma} \label{LemmaNonQuadratic}
With Assumption (\ref{A_Xi2Bound}) (\ref{A_DeltaBound})
(\ref{A_Xi3Bound}) and (\ref{A_DeltaBound2}) , we have
\begin{eqnarray}
t E\|I^{(3)}\|_A^2 \le  \frac{(1+c_0)^2c_4^2}{t}\left((1+2c_2)c_5
\|\Delta_0\|_A^4 + (2c_2  c_3 \|\Delta_0\|_A^2 +(1+2c_2)c_6)
\gamma_1^t + c_3^2 (\gamma_1^t)^2 \right) \nonumber
\end{eqnarray}
where $\gamma_1^t=\sum_{s=1}^t \gamma_s$.
\end{lemma}

\begin{corollary}
For non-quadratic loss, with assumption (\ref{A_Xi2Bound})
(\ref{A_DeltaBound}) (\ref{A_Xi3Bound}) and (\ref{A_DeltaBound2}),
if $c>\frac{1}{2}$, then $\bar{\theta}_t$ satisfies
\[ tE\|\bar{\theta}_t-\theta^*\|_A^2 \le \tr(A^{-1}S) + O\left(t^{-(c-1/2)}\right)+O\left(t^{-(1-c)}\right) \]
\end{corollary}
\begin{proof}
Note that
\[ (E\|\bar{\Delta}_t\|_A^2)^{1/2} \le
(E\|I^{(0)}\|_A^2)^{1/2}+(E\|I^{(1)}\|_A^2)^{1/2}+(E\|I^{(2)}\|_A^2)^{1/2}+(E\|I^{(3)}\|_A^2)^{1/2}
\]
The corollary follows by applying (\ref{EqI0Bound}),
(\ref{EqI1Bound}), Lemma \ref{LemmaRegression} and Lemma
\ref{LemmaNonQuadratic}.
\end{proof}
The best convergence rate is obtained when $c=3/4$, which is
different from that for quadratic loss.
\section{Implementation} \label{SecImplementation}

In this section, we discuss how we implement ASGD for linear
models $f_\theta(x)=\theta^Tx$ with L2 regularization. The running
average can be recursively updated by
$\bar{\theta}_t=(1-\frac{1}{t})\bar{\theta}_{t-1}+\frac{1}{t}\theta_t$,
which is very easy to implement. However, for sparse data sets,
this can be very costly compared to SGD since $\theta_t$ is
typically a dense vector. Consider the following average
procedure:
\begin{eqnarray*}
\theta_t=(1-\lambda \gamma_t) \theta_{t-1} - \gamma_t g_t
\mbox{\quad,\quad} \bar{\theta}_t=(1-\eta_t)\bar{\theta}_{t-1} +
\eta_t \theta_t
\end{eqnarray*}
where $\lambda$ is the L2 regularization coefficient,
$g_t=\pian{L(\theta_{t-1}^T
x_t,y_t)}{\theta_{t_1}}=L_s(\theta_{t-1}^T x_t,y_t) x_t$, and
$\eta_t$ is the rate of averaging. Hence $g_t$ is sparse when
$x_t$ is sparse. We want to take the advantage of the sparsity of
$x_t$ for updating $\theta_t$ and $\bar{\theta}_t$. Let
\[ \alpha_t=\frac{1}{\prod_{i=1}^t (1-\lambda \gamma_i)} \mbox{\quad,\quad} \beta_t=\frac{1}{\prod_{i=1}^t (1-\eta_i)} \mbox{\quad,\quad} u_t=\alpha_t \theta_t \mbox{\quad,\quad} \bar{u}_t=\beta_t \bar{\theta}_t \]
\begin{eqnarray*}
\end{eqnarray*}

After some manipulation, we get the following:
\begin{eqnarray*}
u_t&=&u_{t-1} - \alpha_t \gamma_t g_t \\
\bar{u}_t &=& \bar{u}_{t-1} + \beta_t \eta_t \theta_t = \bar{u}_0 + \sum_{i=1}^t \frac{\beta_i \eta_i}{\alpha_i} u_i \\
&=& \bar{u}_0 + \sum_{i=1}^t \frac{\beta_i \eta_i}{\alpha_i} \left(u_t + \sum_{j=i+1}^t \alpha_j \gamma_j g_j\right) \\
&=& \bar{u}_0 + u_t \sum_{i=1}^t  \frac{\eta_i \beta_i}{\alpha_i} + \sum_{j=1}^t \left(\sum_{i=1}^{j-1} \frac{\eta_i \beta_i}{\alpha_i} \right) \alpha_j \gamma_j g_j\\
\end{eqnarray*}
Now define $\tau_t = \sum_{i=1}^t \frac{\eta_i \beta_i}{\alpha_i}$ and $\hat{u}_t = \hat{u}_{t-1}+ \tau_{t-1} \alpha_t \gamma_t g_t$ with $\hat{u}_0 = \bar{u}_0$, we get
\begin{eqnarray*}
\bar{u}_t &=& \bar{u}_0 + \tau_t u_t + \sum_{j=1}^t \tau_{t-1} \alpha_j \gamma_j g_j =  \tau_t u_t + \hat{u}_t
\end{eqnarray*}
Hence we obtain the following efficient algorithm for updating
$\bar{\theta}_t$:

\begin{algorithm}
\caption{Sparse ASGD} \label{AlgSparseASGD}
\begin{algorithmic}
\STATE $ \alpha_0=1 \mbox{\quad,\quad} \beta_0=1
\mbox{\quad,\quad} \tau_0=0 \mbox{\quad,\quad} u_0=\bar{\theta}_0
\mbox{\quad,\quad} \hat{u}_0=\bar{\theta}_0$

\WHILE{ $t\le T$ }

\STATE $g_t=L_s(\frac{1}{\alpha_{t-1}} u_{t-1}^T x_t,y_t) x_t$

\STATE $\alpha_t=\frac{\alpha_{t-1}}{1-\lambda \gamma_t} $

\STATE $\beta_t=\frac{\beta_{t-1}}{1-\eta_t}$

\STATE $u_t=u_{t-1} - \alpha_t \gamma_t g_t$

\STATE $\hat{u}_t=\hat{u}_{t-1}+\tau_{t-1} \alpha_t \gamma_t g_t$

\STATE $\tau_t=\tau_{t-1}+\frac{\eta_t \beta_t}{\alpha_t} $

\ENDWHILE
\end{algorithmic}
\end{algorithm}
At any step of the algorithm, $\bar{\theta}_t$ can be obtained by
$\bar{\theta}_t=\frac{\bar{u}_t}{\beta_t}=\frac{\tau_t
u_t+\hat{u}_t}{\beta_t}$. Note that in Algorithm
\ref{AlgSparseASGD}, none of the operations involves two dense
vectors. Thus the number of operations per sample is $O(Z)$,
where $Z$ is the number of non-zero elements in $x$.

From Theorem \ref{ThmLinearAvg} we can see that if
$\|\Delta_0\|_{A^{-1}}^2$ is large compared to $\tr(A^{-1}S)$,
then the error is dominated by $I^{(0)}$ at the beginning. This
can happen if noise is small compared to $\|\Delta_0\|$. It is
possible to further improve the performance of ASGD by discarding
$\theta_t$ from averaging during the initial period of training.
We want to find a point $t_0$ whereafter averaging becomes
beneficial. For this, we maintain an exponential moving average
$\hat{\theta}_t=0.99 \hat{\theta}_{t-1}+ 0.01 \theta_t$ and
compare the moving average of the empirical loss of
$\hat{\theta}_t$ and $\theta_t$. Once $\hat{\theta}_t$ is better
than $\theta_t$, we begin the ASGD procedure.

\section{Experiments} \label{SecExperiments}
In this section, we provide 3 sets of experiments. The first
experiment illustrate the importance of learning rate scheduling
for ASGD. The second experiment illustrates the asymptotic optimal
convergence of ASGD. In the third set of experiments, we apply
ASGD on many public benchmark data sets and compare it with
several state of the art algorithms.

\subsection{Effect of learning rate scheduling}
Our first experiment is used to show how different learning rate
schedule affects the convergence of ASGD using a synthetic
problem. The exemplar optimization problem is $\min_\theta E_x
((\theta-x)^TA(\theta-x))$, where $A$ is a symmetric 100x100
matrix with eigenvalues $[1, 1, 1, 0.02 \cdots 0.02]$ and $x$
follows normal distribution with zero mean and unit covariance. It
can be shown that the optimal $\theta$ is $\theta^*=0$. Figure
\ref{FigToy3} shows the excess risk
$\mathcal{E}(\theta_t)-\mathcal{E}(\theta^*)$ of the solution vs.
number of training samples $t$. We note that in this particular
example the excess risk is simply $\theta_t^T A \theta_t$. For the
good example of ASGD (ASGD in the figure), we use our proposed
learning rate schedule $\gamma_t=(1+ 0.02 t)^{-2/3}$ according to
Section \ref{SecRegression}. For a bad example of ASGD (ASGD\_BAD
in the figure), we use $\gamma_t=(1+t)^{-1/2}$, which looks simple
and also has optimal asymptotic convergence according to Corollary
\ref{CoralloryLinearAymptoticConvergence}. Figure \ref{FigToy3}
also shows the performance of standard SGD using learning rate
schedule $\gamma_t=(1+0.02 t)^{-1}$ and batch method
$\theta_t=\frac{1}{t}\sum_{j=1}^t x_t$. We see that both ASGD and
ASGD\_BAD eventually outperforms SGD and come close to the batch
method. However, it takes only a few thousands example for ASGD to
get to the asymptotic region, while it takes hundreds of thousands
of examples for ASGD\_BAD. This huge difference illustrates the
significant role of learning rate scheduling for ASGD.

\begin{figure}[ht]
  \centering
\epsfxsize=10cm \epsfysize=5.5cm
\epsfbox{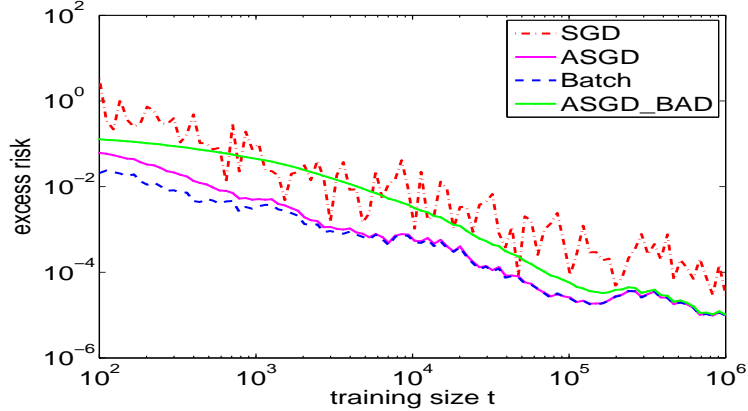}
\caption{\label{FigToy3} ASGD with proposed learning rate schedule (ASGD) and an arbitrarily chosen learning rate schedule (ASGD\_BAD).}
\end{figure}

\subsection{Asymptotic optimal convergence}

Our second experiment is used to show the asymptotic optimality of
ASGD for linear regression. For this purpose, we generate
synthetic regression problem $y=x^T \theta^*+\epsilon$, where $x$
is $N=100$ dimensional vector following Gaussian distribution with
zero mean and covariance $A$, the eigenvalues of $A$ are evenly
spread from 0.01 to 1, $\theta^*$ is a vector with all dimension
equal to 1, $\epsilon$ follows Gaussian distribution with zero
mean and unit variance. We compare ASGD with SGD and batch method.
We use $\gamma_0=1/\tr(A)$ for both ASGD and SGD. For batch
method, we simply calculate $\theta_t$ as $\theta_t=(\sum_{i=1}^t
x_i x_i^T)^{-1} \sum_{i=1}^t x_i y_i$. Figure \ref{FigToy} shows
the excess risk $\mathcal{E}(\theta_t)-\mathcal{E}(\theta^*)$ of
the solution vs. number of training samples $t$. As the figure
shows, after about $10^4$ examples, the accuracy of ASGD starts to
be close to batch solution while the solution of SGD remains more
than 10 times worse than ASGD. Note that although ASGD and batch
solution has similar accuracy, ASGD is considerably fast than
batch method since ASGD only need $O(N)$ computation per sample
while batch method need $O(N^2)$ computation per sample.
\begin{figure}[ht]
  \centering
\epsfxsize=10cm \epsfysize=5.5cm
\epsfbox{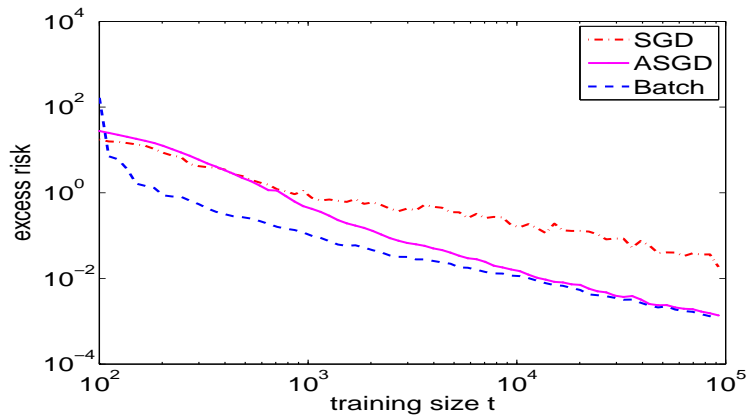}
\caption{\label{FigToy} Compare ASGD with batch method.}
\end{figure}

\subsection{Experiments on benchmark data sets}
In the third set of experiments, we compare ASGD with several
other algorithms for training large scale linear models: online
limited-memory BFGS (oLBFGS) of \cite{Schraudolph07}, stochastic
gradient descent (SGD2) of \cite{Bottou07}, dual coordinate
descent (LIBLINEAR) of \cite{Fan08}, Pegasos of
\cite{Shalev-Shwartz07} and SGDQN of \cite{Bordes09}. We performed
extensive evaluation of ASGD on many data sets. Due to space
limit, we only show detailed results on four tasks in this paper.
COVTYPE is the detection of class 2 among 7 forest cover types
(Blackard et al). All dimensions are normalized between 0 and 1.
DELTA is a synthetic data set from the PASCAL Large Scale
Challenge \citep{Sonnenburg08}. We use the default data
preprocessing provided by the challenge organizers. RCV1 is the
classification of documents belonging to class CCAT in RCV1 text
data set \citep{Lewis04}. We use the same preprocessing as
provided in \cite{Bottou07}. MNIST9 is the classification of digit
9 against all other digits in MNIST digit image data set
\citep{LeCun98}. For this task, we generate our own image feature
vectors for recognition. The experiments for these four tasks use
squared hinge loss $L(s,y)=\frac{1}{2}(\max(0,1-ys))^2$ with $L2$
regularization $R(\theta)=\frac{\lambda}{2}\|\theta\|_2^2$. Since
$\lambda_0$ is unknown, we use the regularization coefficient
$\lambda$ as $\lambda_0$, which is a lower bound for true
$\lambda_0$. Table \ref{TblData} summarizes the data sets, where
$M$ is the $\max \|x\|^2$ calculated from 1000 samples, $t_0$ is
the point where average begins (See Section
\ref{SecImplementation}). Figure \ref{FigComparison} shows the
test error rate (left), elapsed time (middle) and test cost
(right) at different points within first two passes of training
data.

We also include more experimental results on data sets from Pascal
Large Scale Challenge. However, to save space, we only show
figures for test error rate. All experiments use the default data
preprocessing provided by the challenge organizers. Table
\ref{TblData2} summarize the data sets. Figure \ref{FigSet1} and
Figure \ref{FigSet2} shows result for L2 SVM, logistic regression
and SVM. LIBLINEAR is not included in the figures for logistic
regression because the dual coordinate descent method used by
LIBLINEAR cannot solve logistic regression. Although the theory of
ASGD only applies to smooth cost functions, we also include the
results of SVM to satisfy the possible curiosity of some readers.

As we can see from the figures, ASGD clearly outperforms all other
5 algorithms in terms accuracy in most of the data sets. In fact,
for most of the data sets, ASGD reaches good performance with only
one pass of data, while many other algorithms still perform poorly
at that point. The only exception is the \emph{beta} data set,
where all methods performs equally bad because the two classes in
this data set are not linearly separable. Moreover, the
performance of the other 5 methods tend to be more volatile, while
performance of ASGD is more robust due to average. In terms of
time spent on one pass of data, ASGD is similar to the other
methods except oLBFGS, which means that ASGD needs less time to
reach similar test performance compared to the other methods.
Another interesting point is that although the current theory of
ASGD is based on the assumption that cost function is smooth, as
shown in the figures, ASGD also works pretty well with non-smooth
loss such as hinge loss.
\begin{table}[htbp]
  \centering
\caption{\label{TblData} Data Set Summary}
\begin{tabular}{|l|l|l|l|l|l|l|l|l|}
\hline
       & description          & type   & dim & train size & test size & $\lambda$   & $M$ & $t_0$ \\
\hline
covtype& forest cover type    & sparse & 54        & 500k       & 81k     & $10^{-6}$ & 6.8 & 100 \\
delta  & synthetic data       & dense  & 500       & 400k       & 50k    & $10^{-2}$ & $3.8\times 10^3$ & 100 \\
rcv1   & text data            & sparse & 47153     & 781k       & 23k    & $10^{-5}$ & 1 & 781 \\
mnist9 & digit image features & dense  & 2304      & 50k        & 10k    & $10^{-3}$ & $2.1\times 10^4$ & 128 \\
\hline
\end{tabular}
\end{table}

\begin{figure}[ht]
\begin{tabular}{ccc}
\epsfxsize=4.7cm \epsfysize=4.4cm
\epsfbox{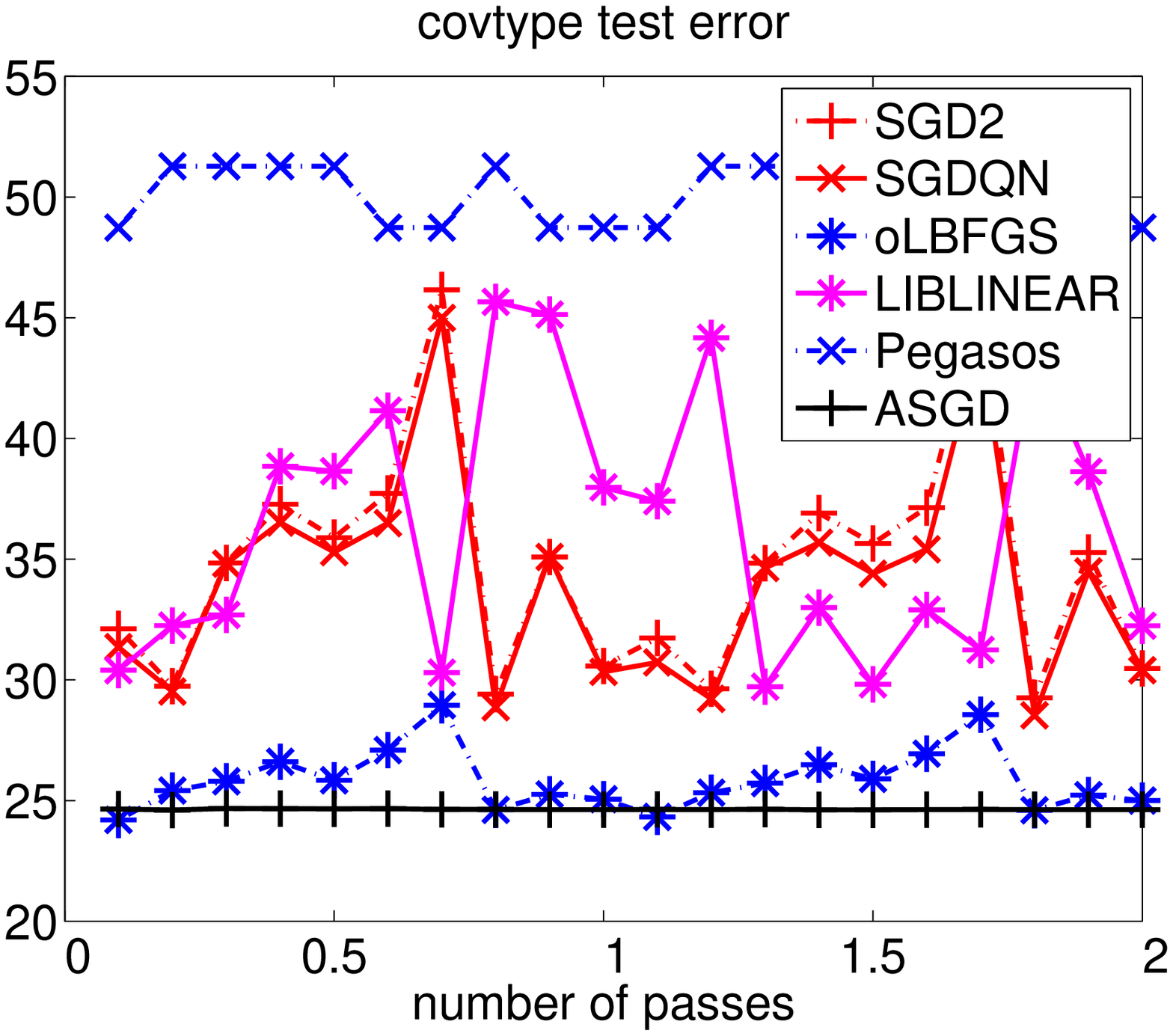} &
\epsfxsize=4.7cm \epsfysize=4.4cm
\epsfbox{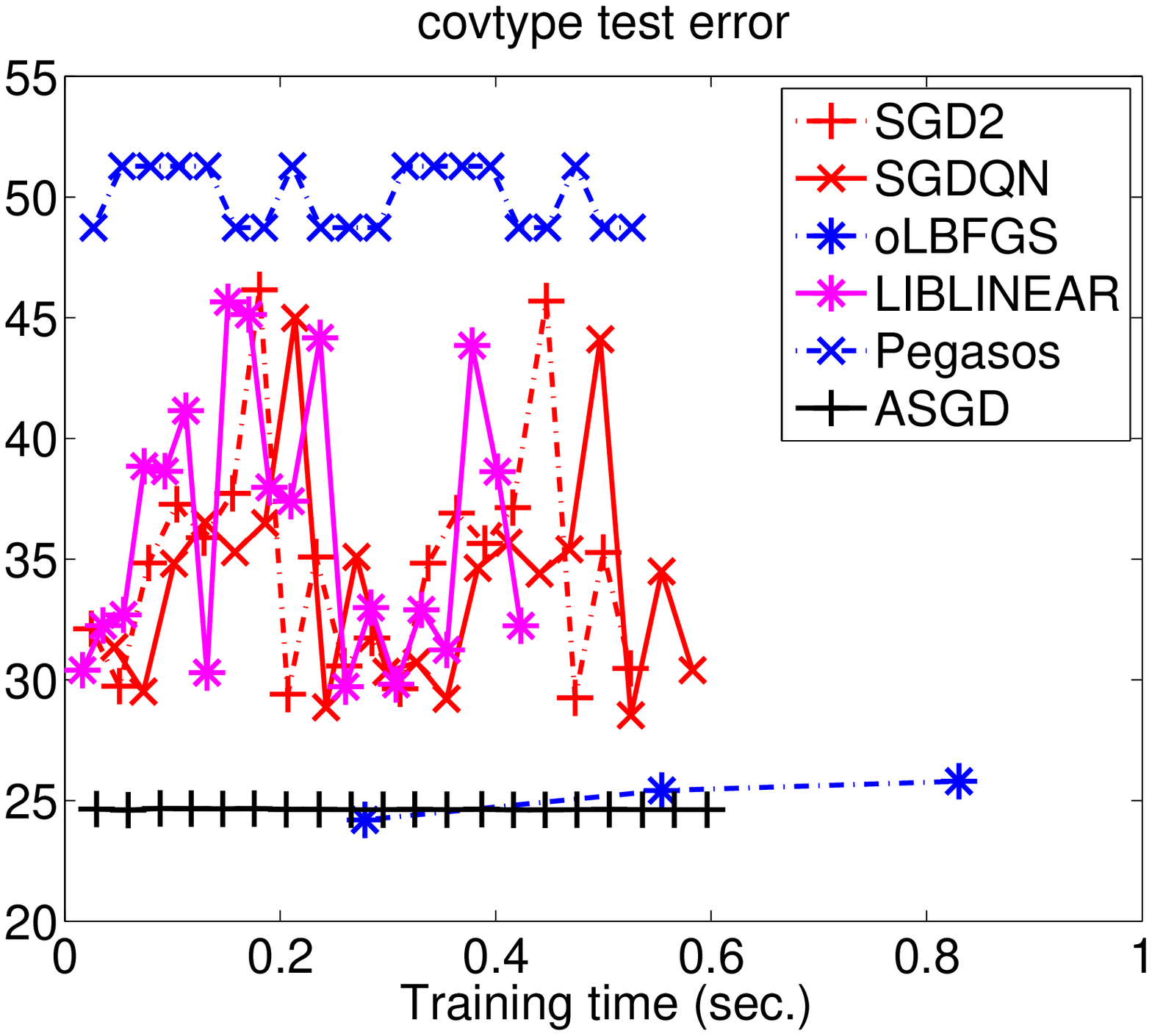} &
\epsfxsize=4.7cm \epsfysize=4.4cm
\epsfbox{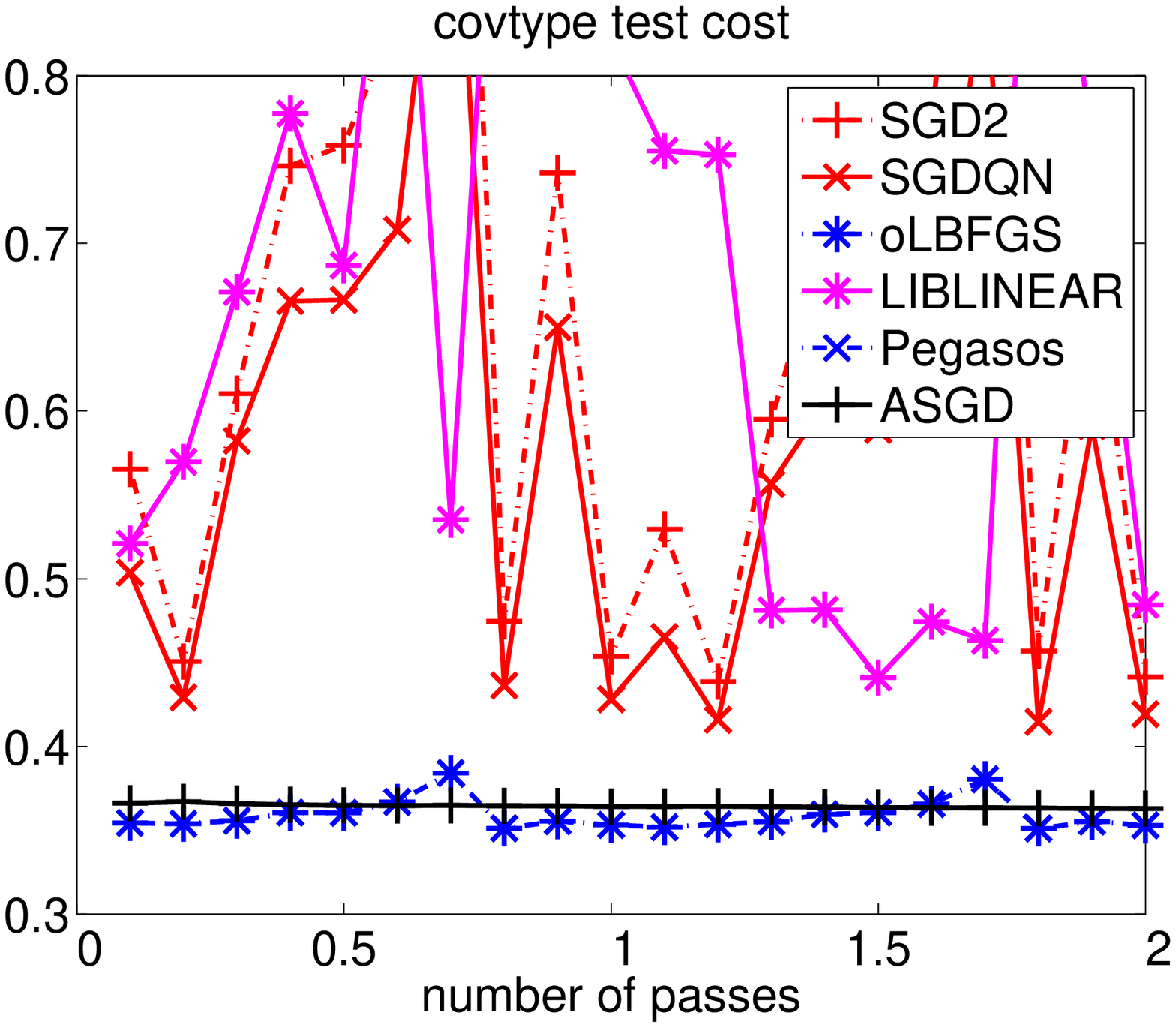} \\

\epsfxsize=4.7cm \epsfysize=4.4cm
\epsfbox{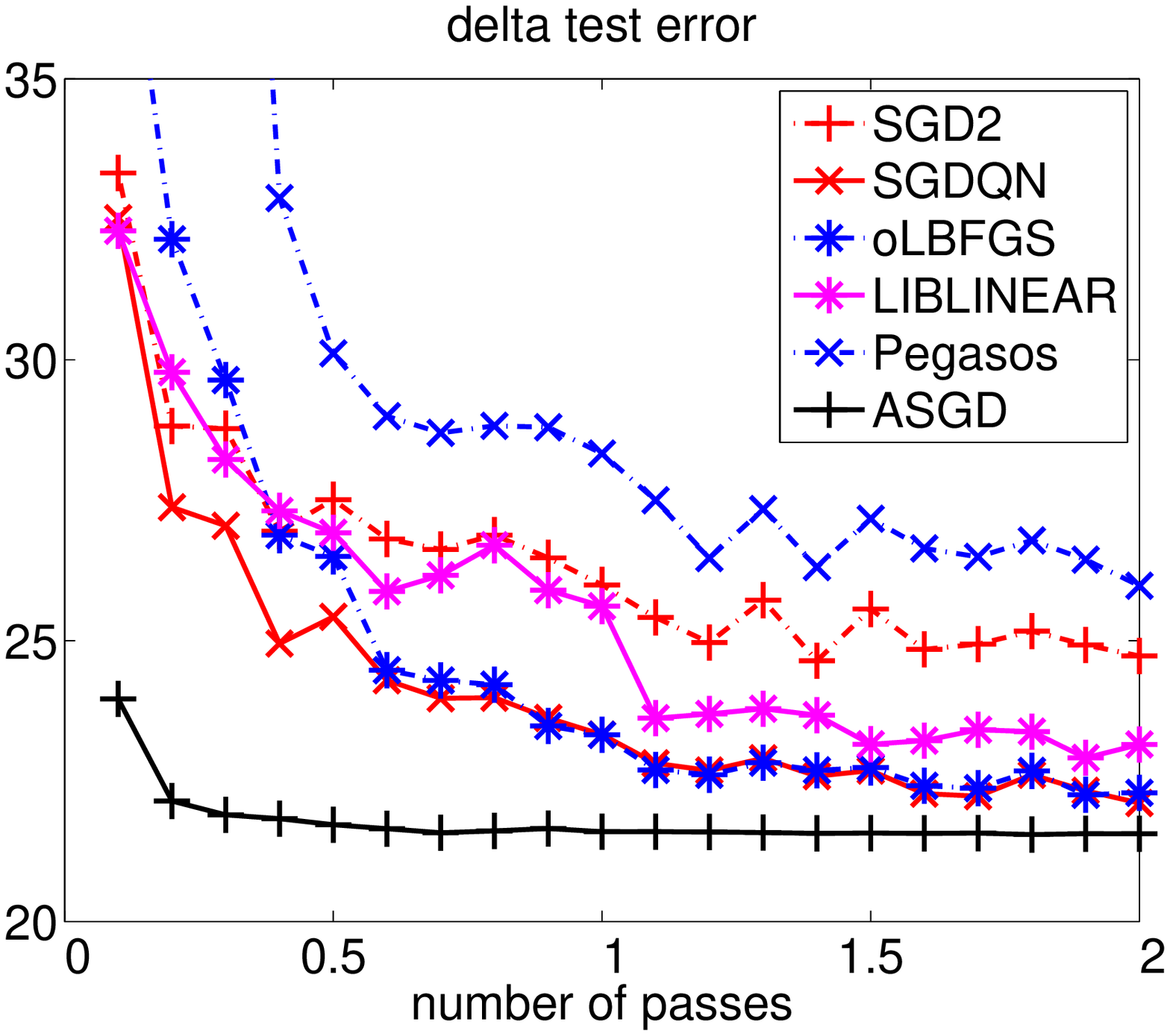} &
\epsfxsize=4.7cm \epsfysize=4.4cm
\epsfbox{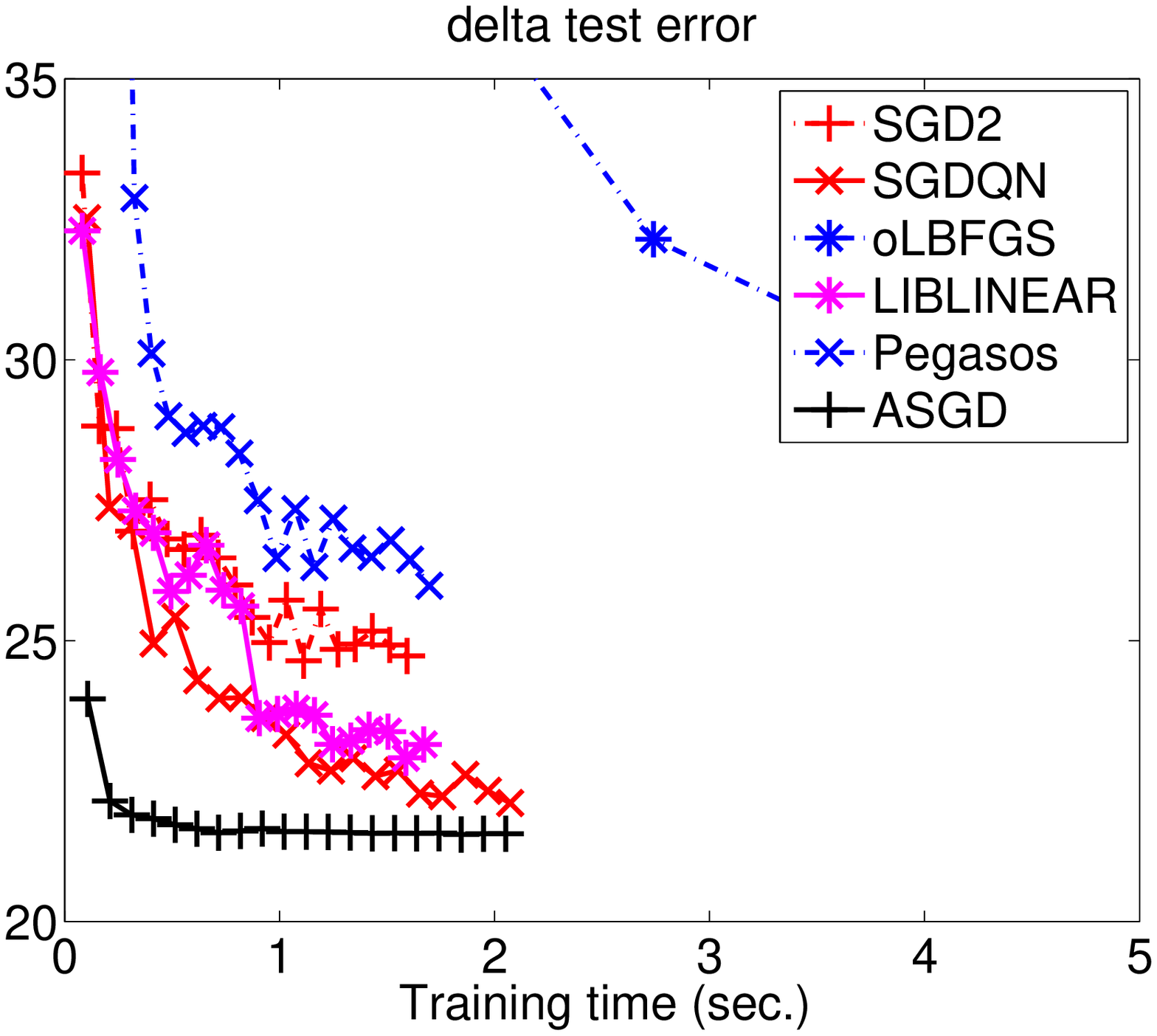} &
\epsfxsize=4.7cm \epsfysize=4.4cm
\epsfbox{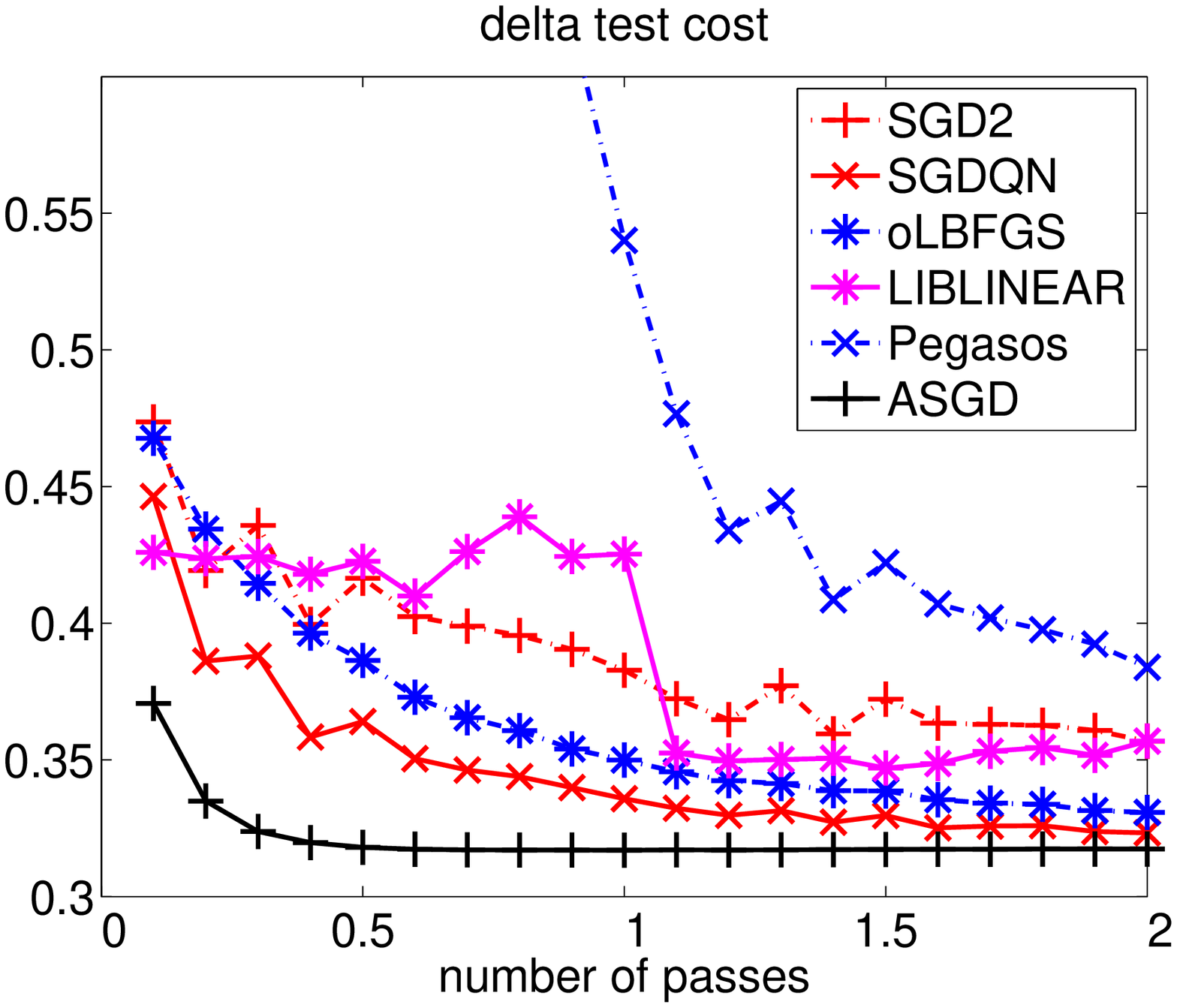} \\

\epsfxsize=4.7cm \epsfysize=4.4cm
\epsfbox{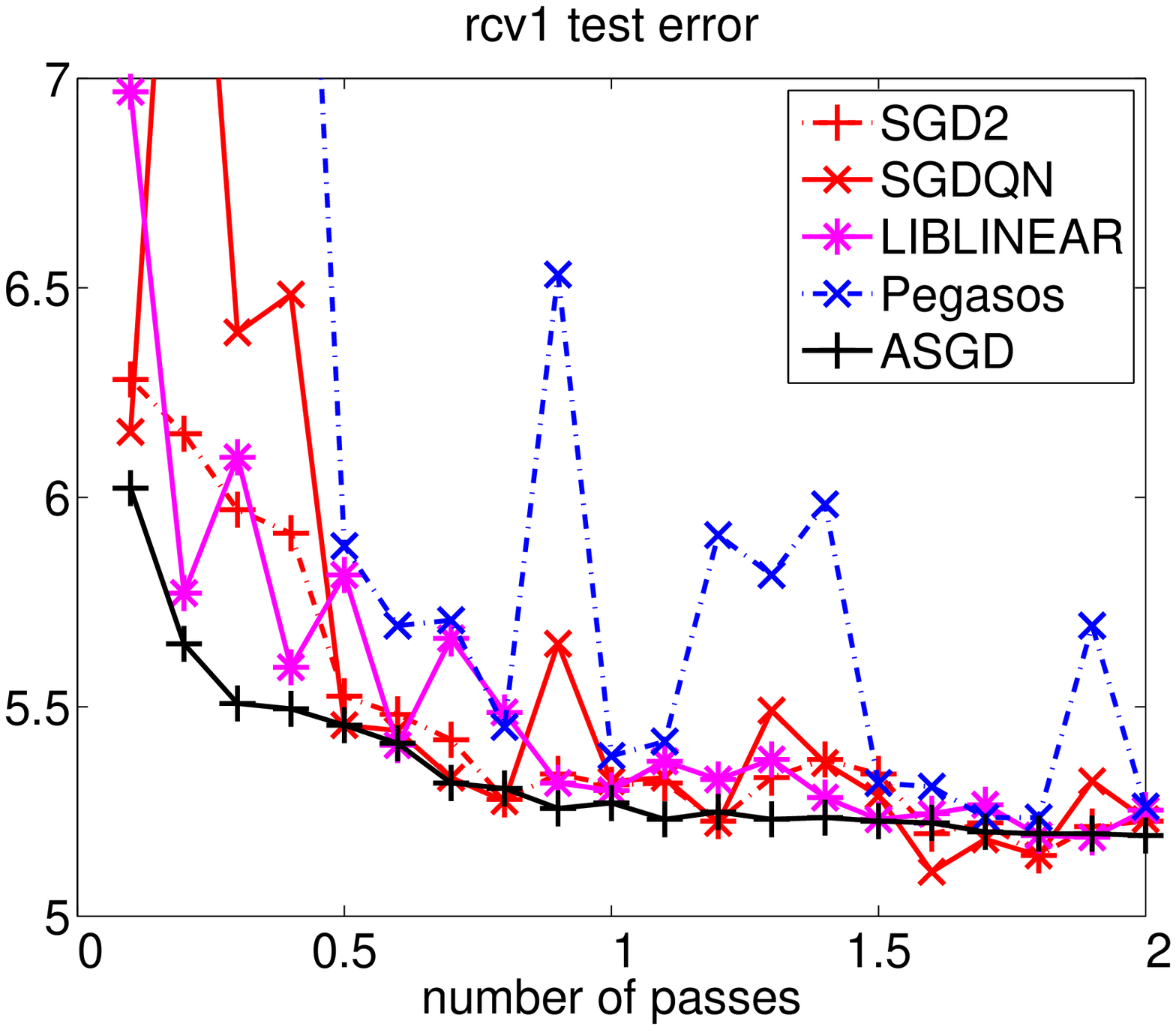} &
\epsfxsize=4.7cm \epsfysize=4.4cm
\epsfbox{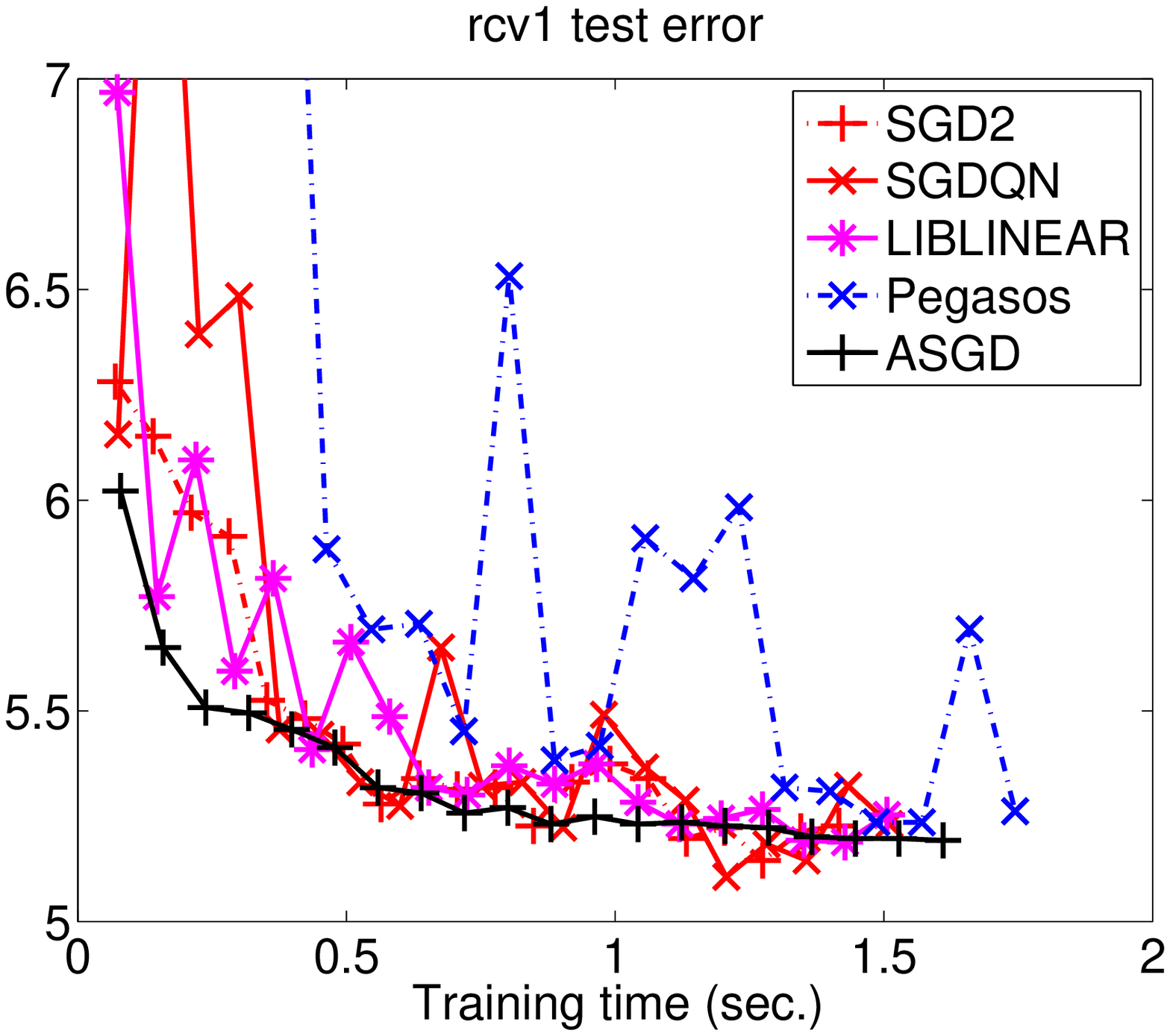} &
\epsfxsize=4.7cm \epsfysize=4.4cm
\epsfbox{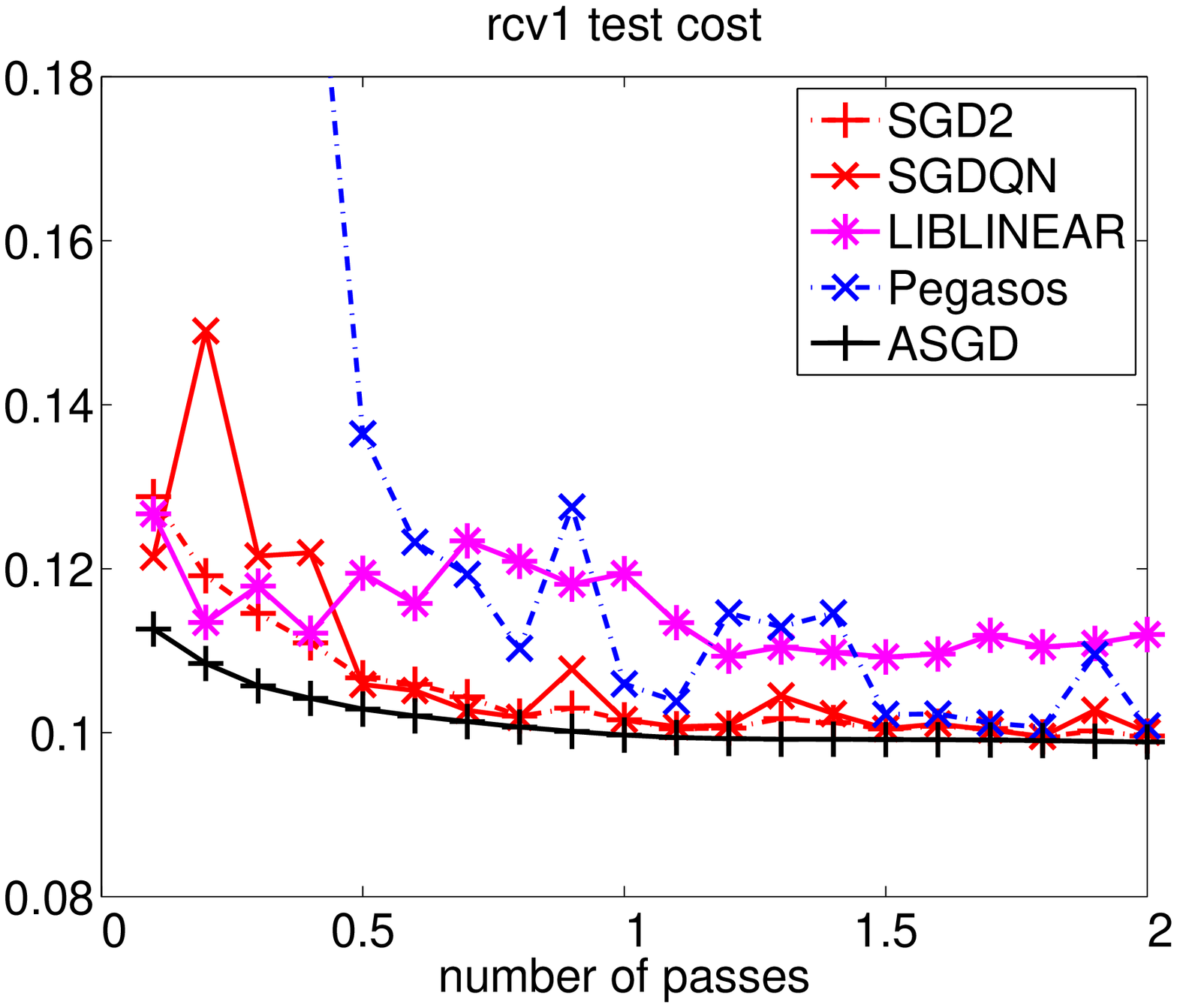} \\

\epsfxsize=4.6cm \epsfysize=4.4cm
\epsfbox{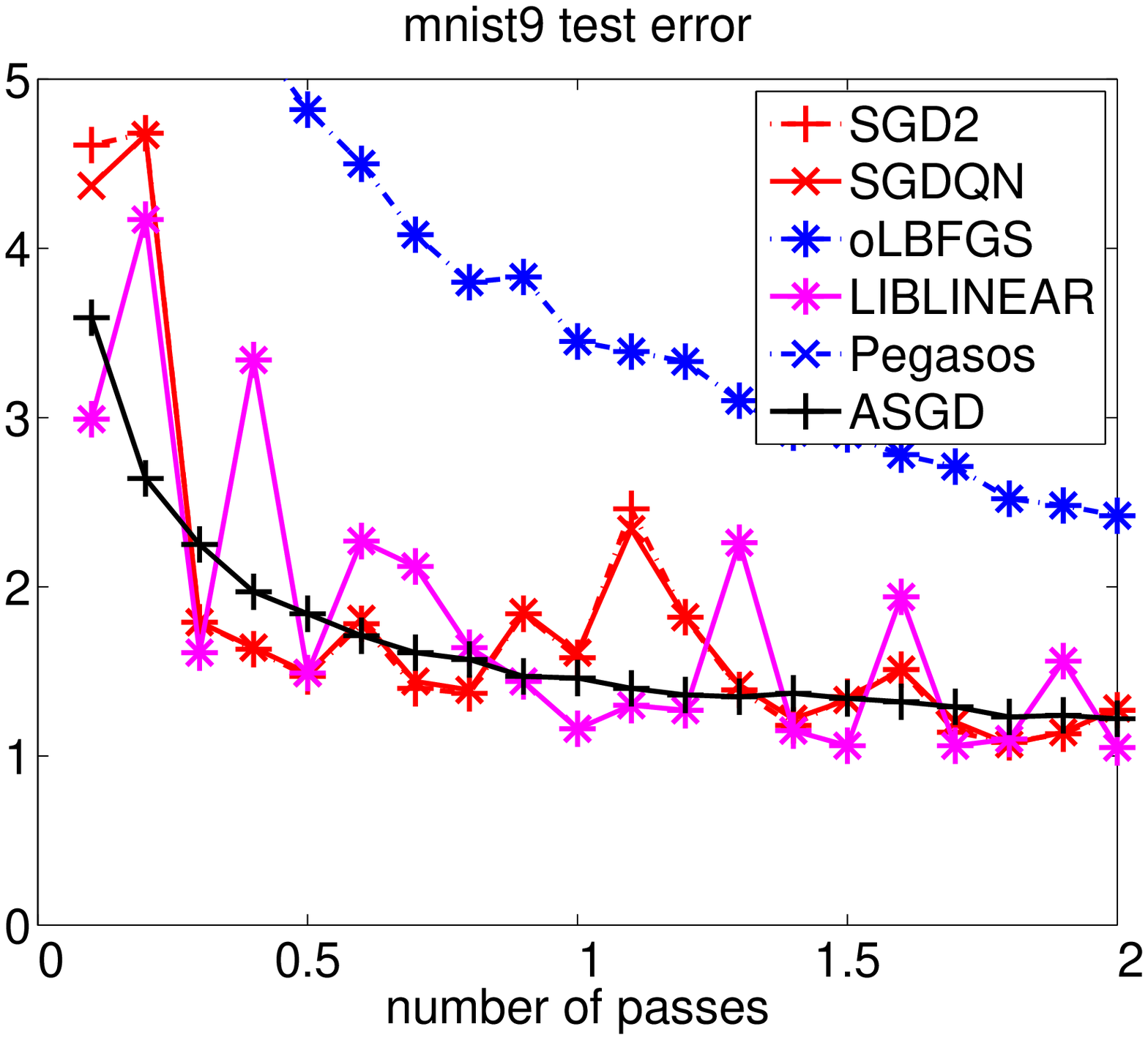} &
\epsfxsize=4.6cm \epsfysize=4.4cm
\epsfbox{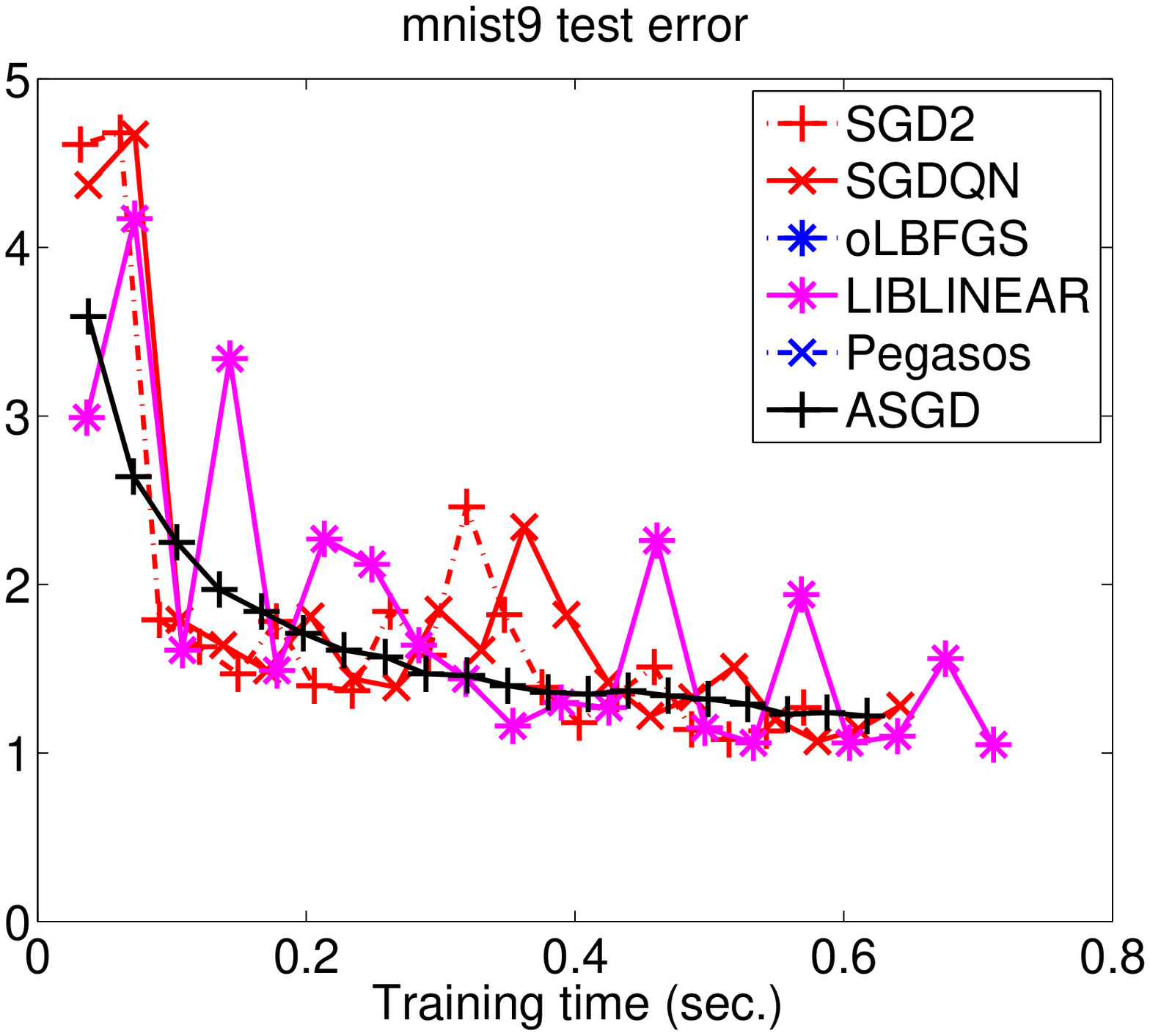} &
\epsfxsize=4.6cm \epsfysize=4.4cm
\epsfbox{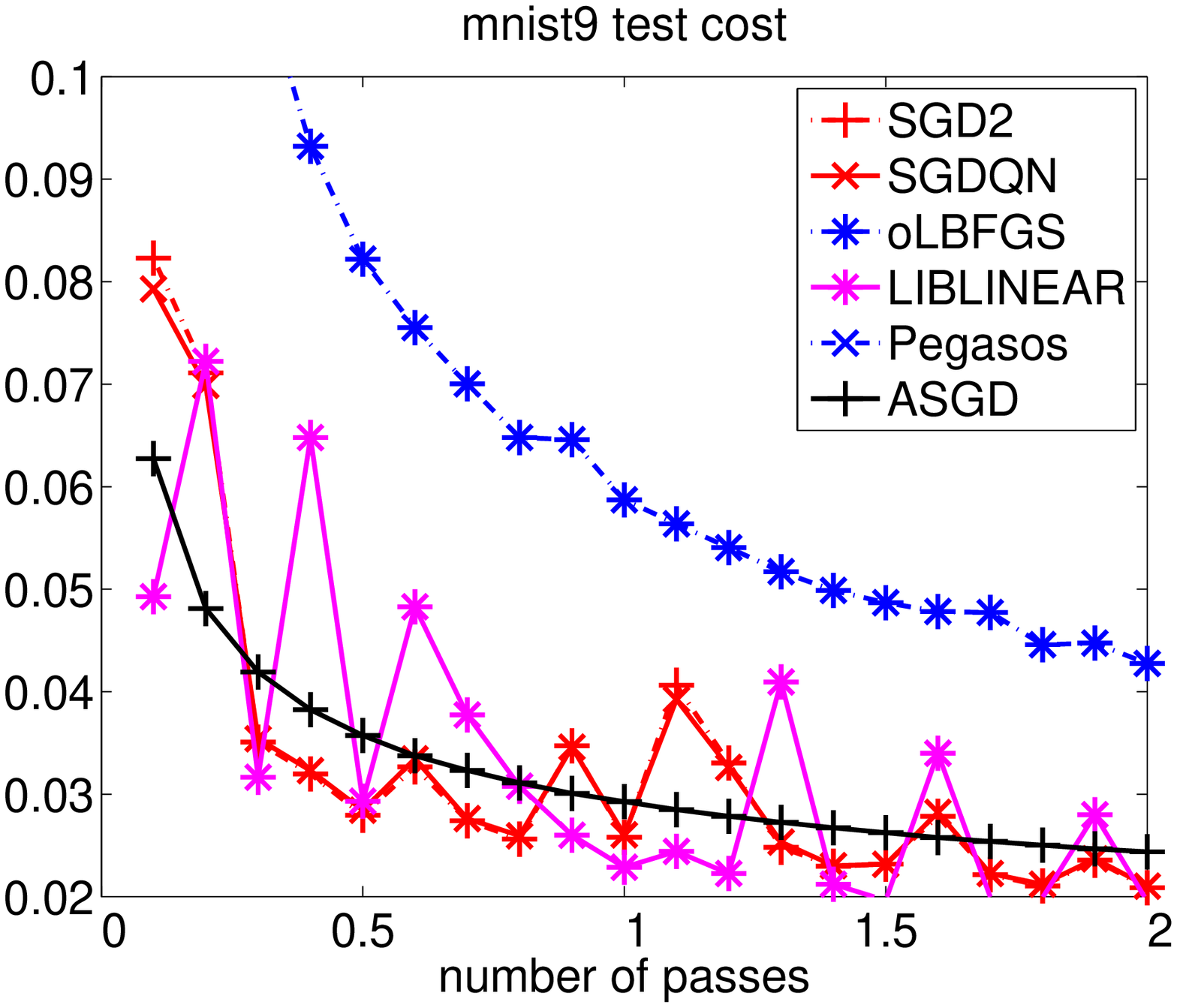} \\
\end{tabular}
\caption{\label{FigComparison} Left: Test error (\%) vs. number of passes. Middle: Test error vs. training time. Right: Test cost vs. number of passes.}
\end{figure}

\begin{table}[htbp]
  \centering
\caption{\label{TblData2} Data Set Summary}
\begin{tabular}{|l|l|l|l|l|l|l|l|}
\hline
       & description          & type   & dim & train size & test size & $\lambda$   & $M$ \\
\hline
alpha  & synthetic data       & dense  & 500       & 400k       & 50k    & $10^{-5}$ & 1 \\
beta   & synthetic data       & dense  & 500       & 400k       & 50k    & $10^{-4}$ & 1 \\
gamma  & synthetic data       & dense  & 500       & 400k       & 50k    & $10^{-3}$ & $2.5\times 10^3$ \\
epsilon& synthetic data       & dense  & 2000      & 400k       & 50k    & $10^{-5}$ & 1   \\
zeta   & synthetic data       & dense  & 2000      & 400k       & 50k    & $10^{-5}$ & 1   \\
fd         & character image  & dense  &   900 & 1000k &  470k & $10^{-5}$ & 1        \\
ocr        & character image  & dense  &  1156 & 1000k &  500k & $10^{-5}$ & 1        \\
dna        & DNA sequence     & sparse &   800 & 1000k & 1000k & $10^{-3}$ & 200    \\
\hline
\end{tabular}
\end{table}


\begin{figure}[ht]
\begin{tabular}{ccc}
\epsfxsize=4.7cm \epsfysize=4.5cm
\epsfbox{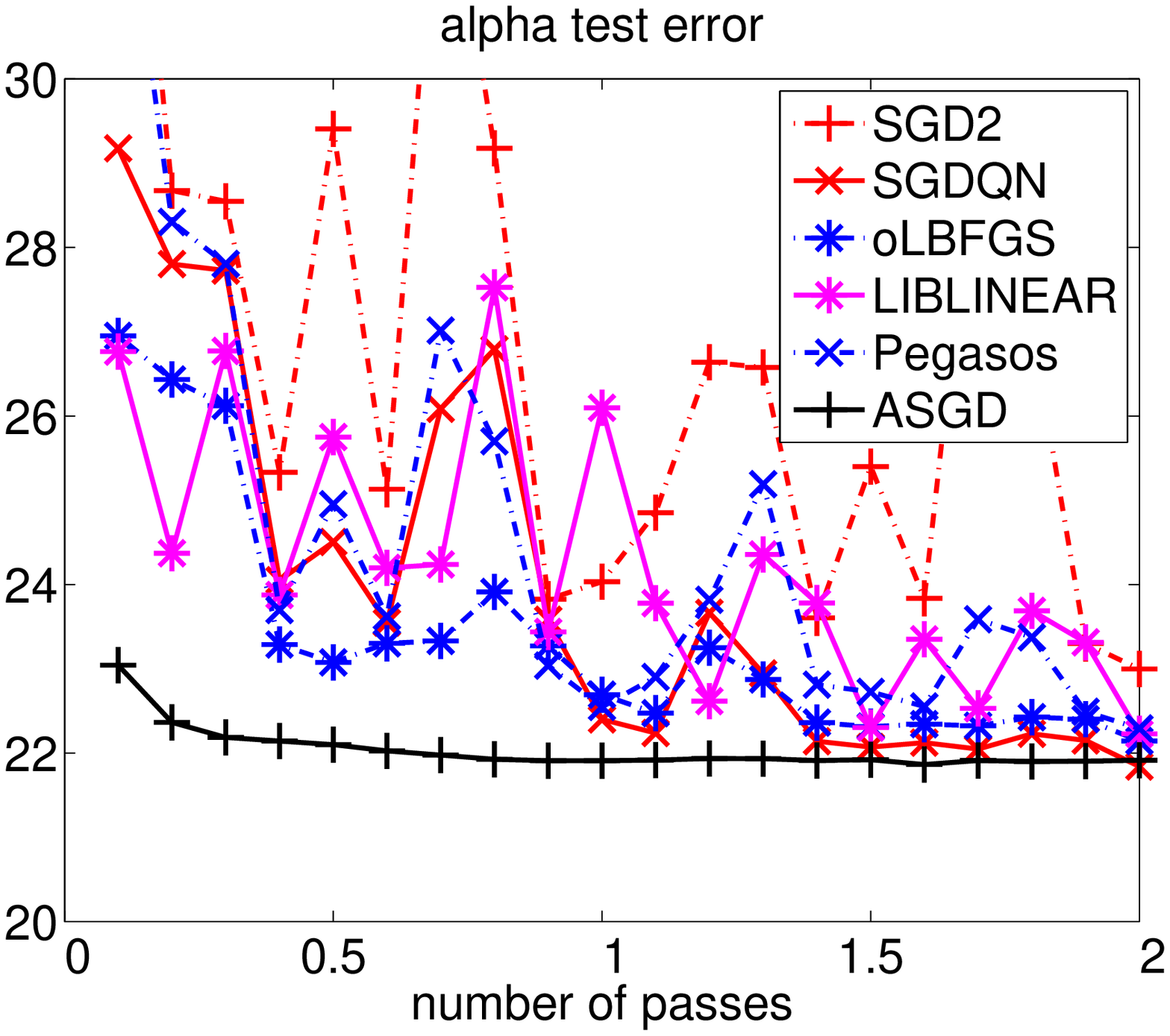} &
\epsfxsize=4.7cm \epsfysize=4.5cm
\epsfbox{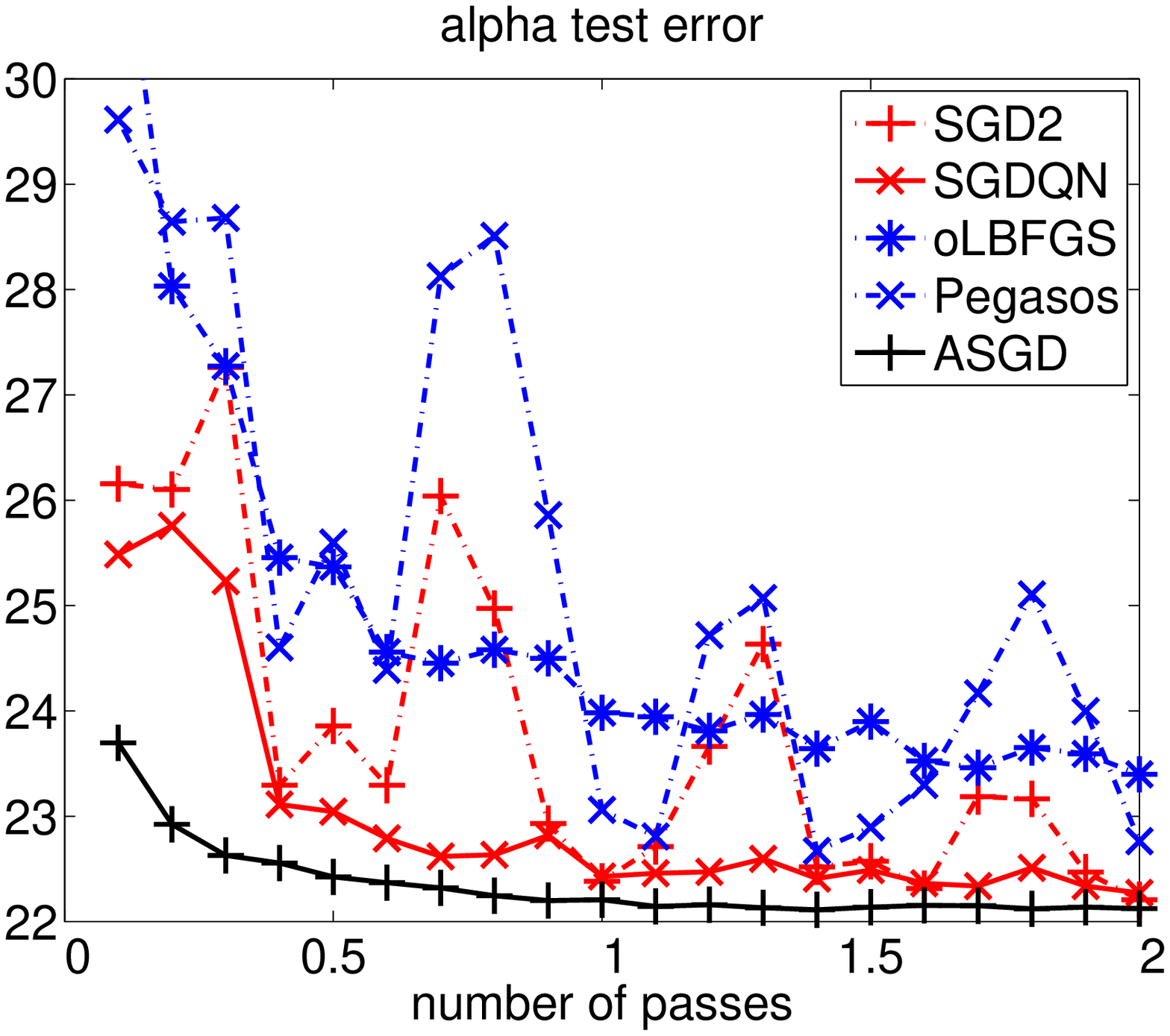} &
\epsfxsize=4.7cm \epsfysize=4.5cm
\epsfbox{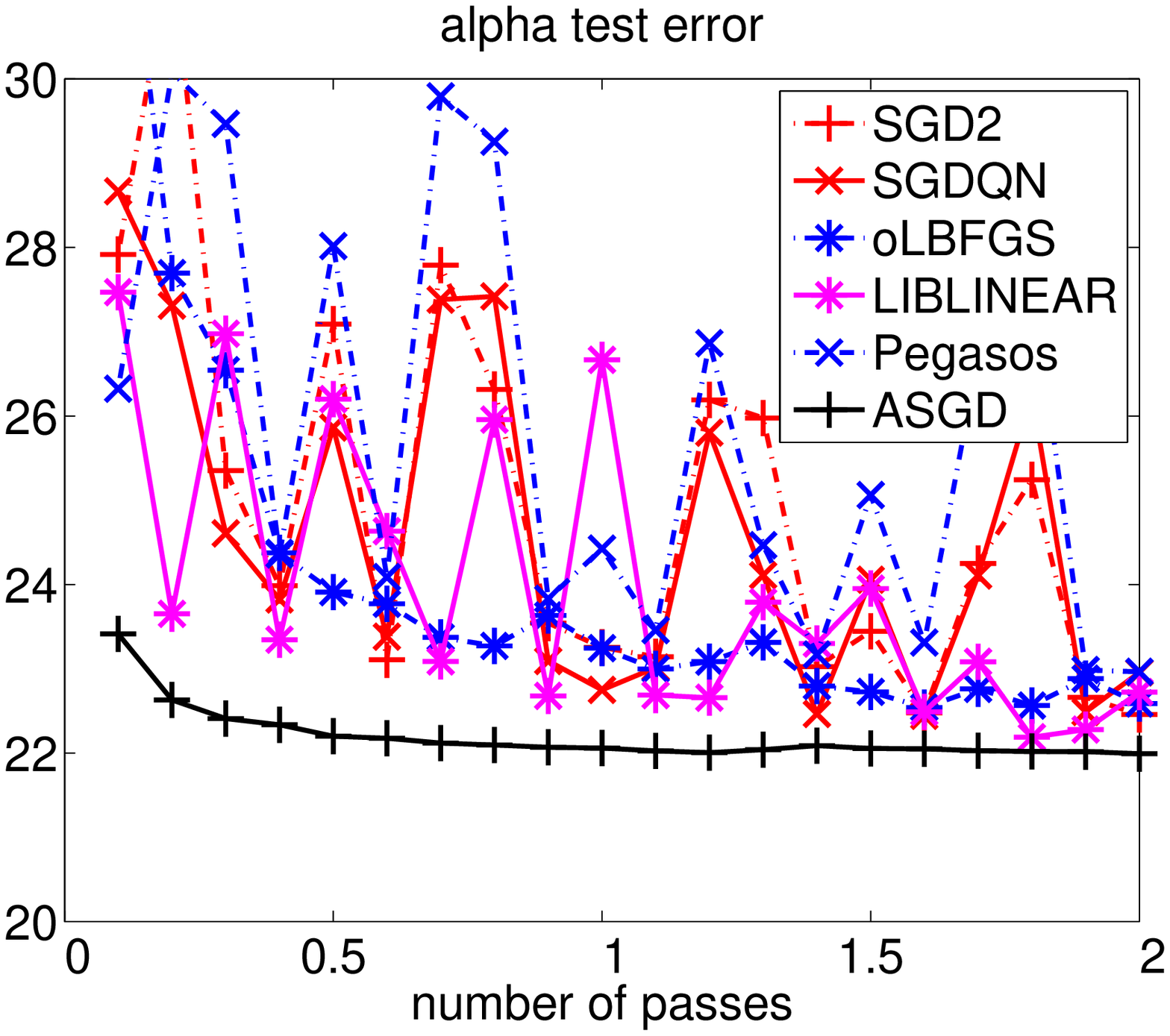} \\

\epsfxsize=4.7cm \epsfysize=4.5cm
\epsfbox{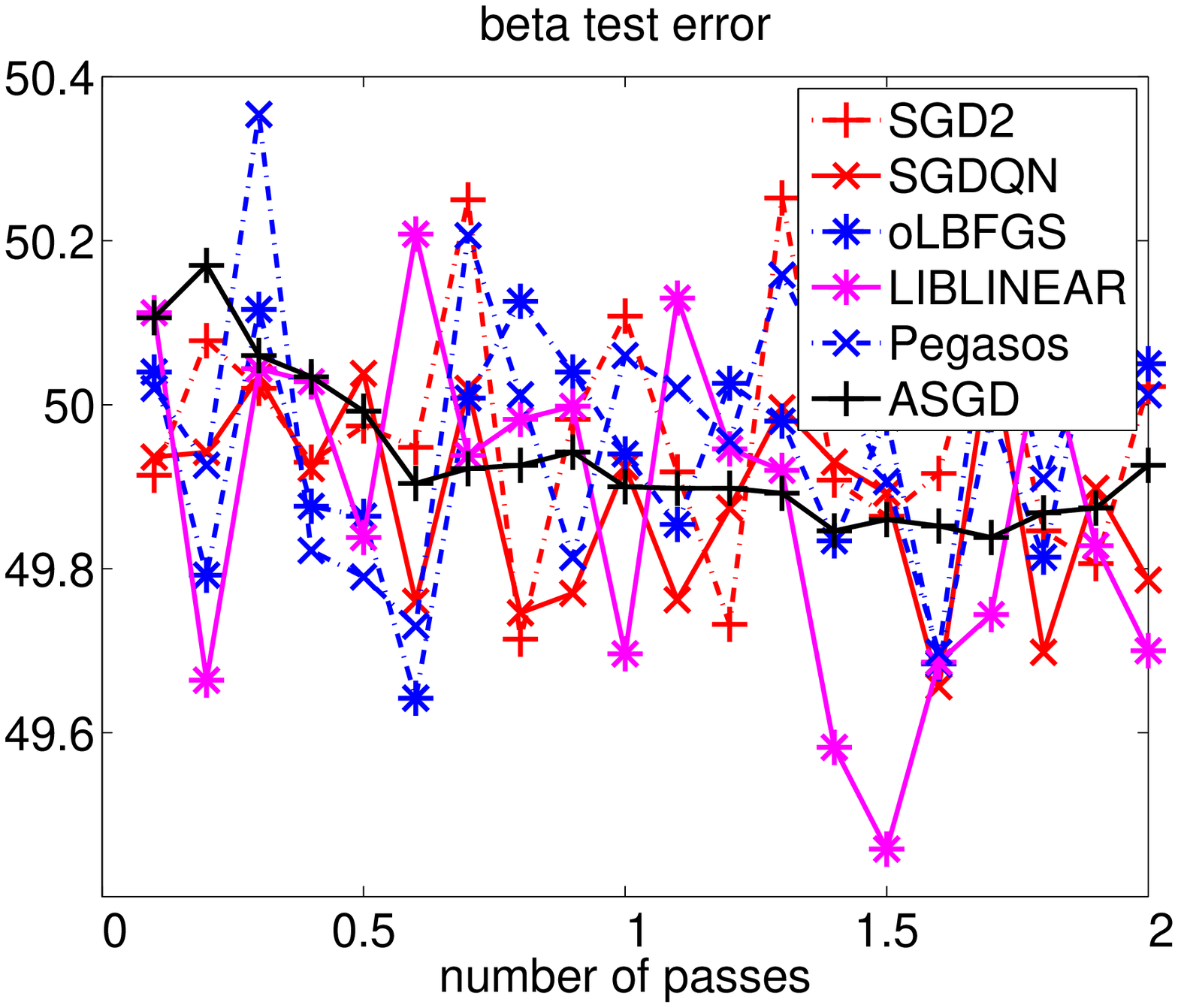} &
\epsfxsize=4.7cm \epsfysize=4.5cm
\epsfbox{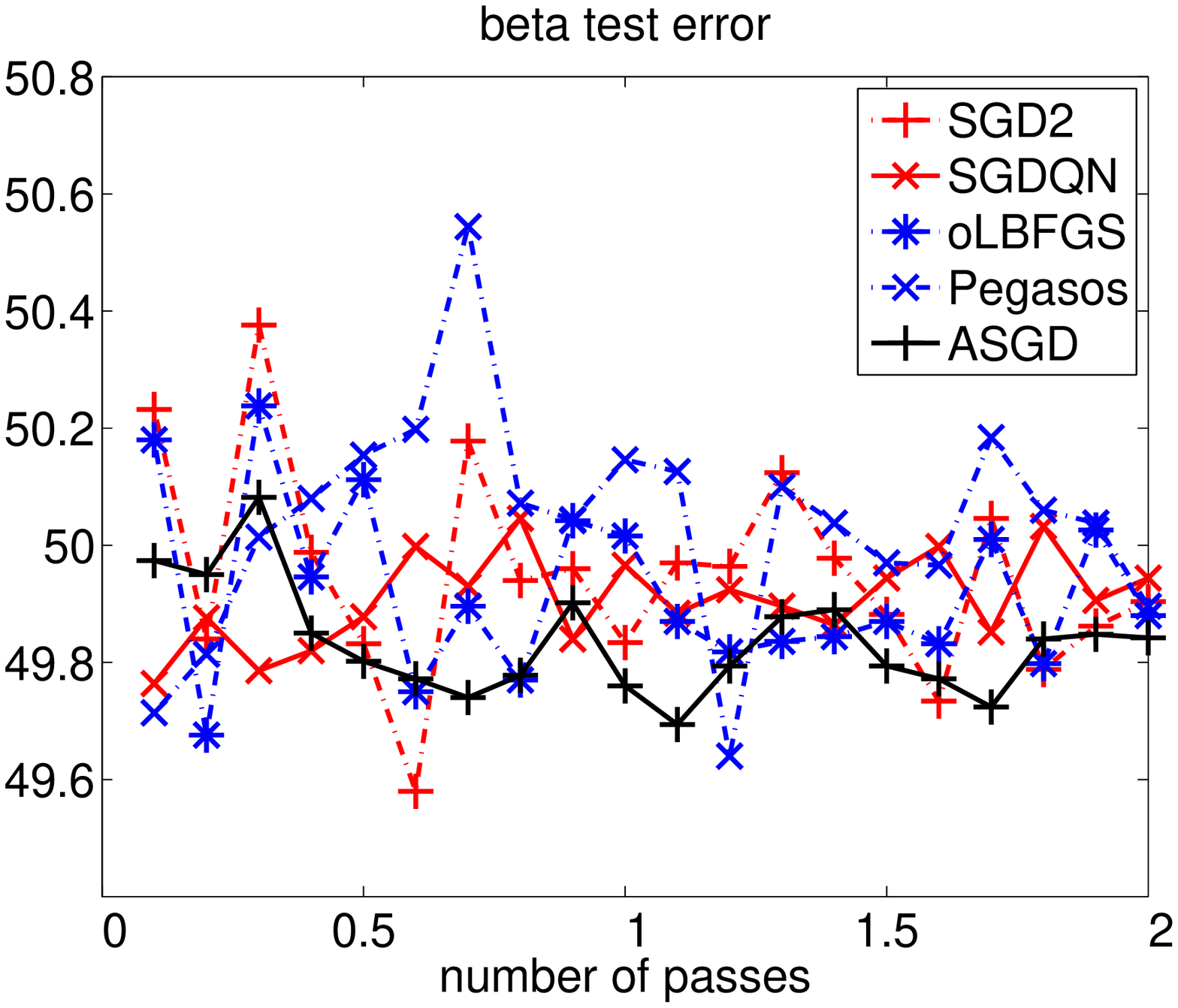} &
\epsfxsize=4.7cm \epsfysize=4.5cm
\epsfbox{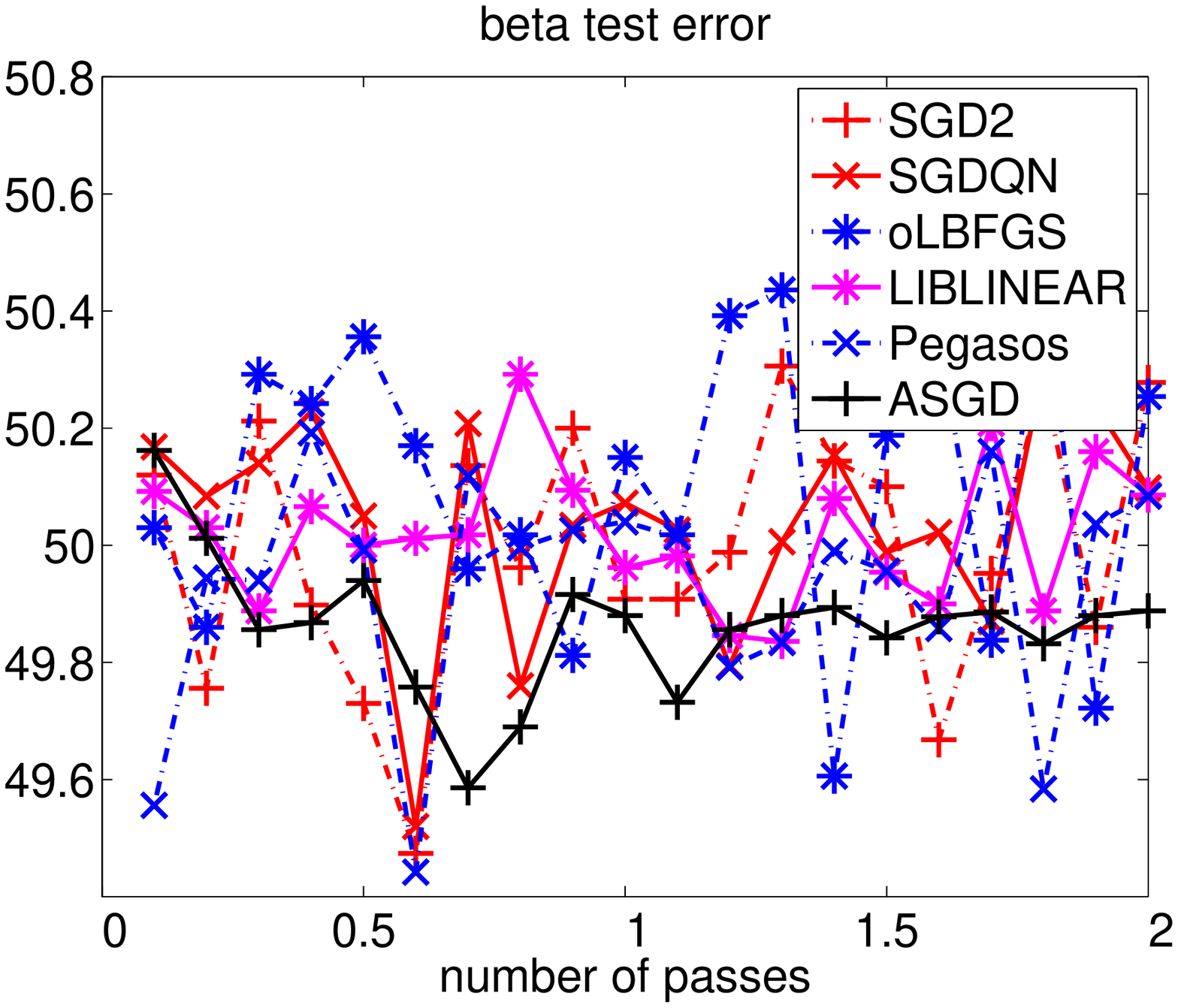} \\

\epsfxsize=4.7cm \epsfysize=4.5cm
\epsfbox{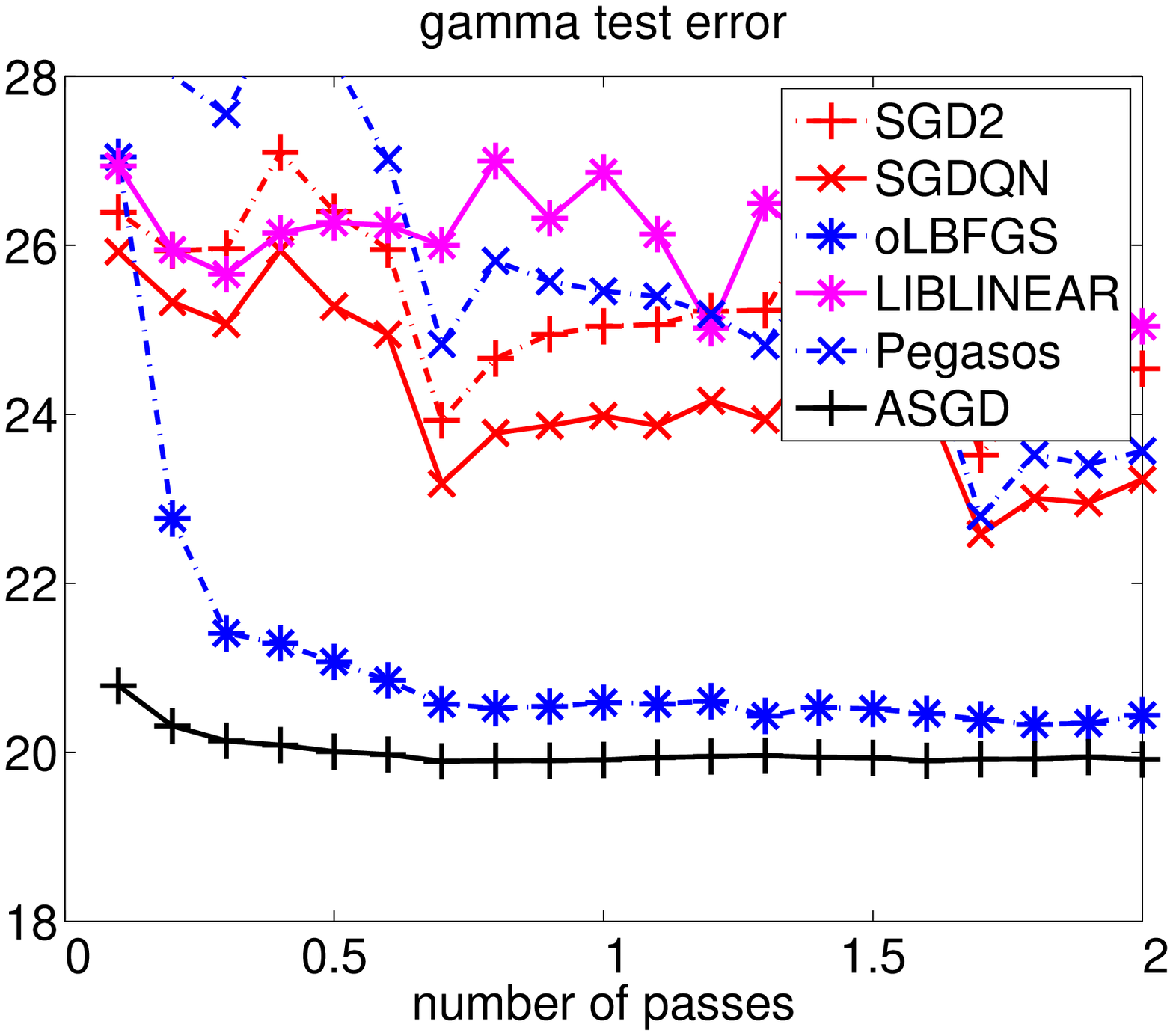} &
\epsfxsize=4.7cm \epsfysize=4.5cm
\epsfbox{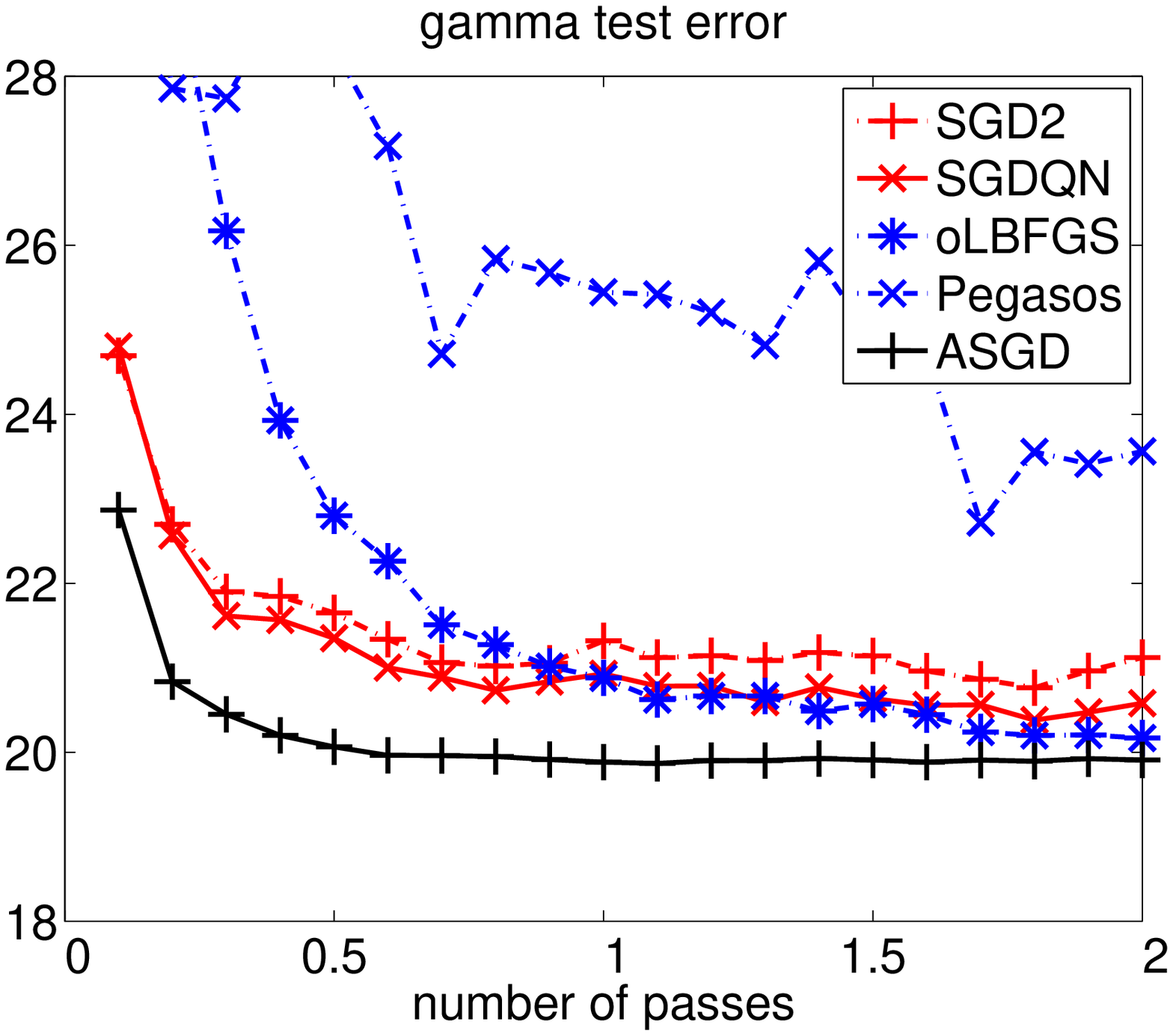} &
\epsfxsize=4.7cm \epsfysize=4.5cm
\epsfbox{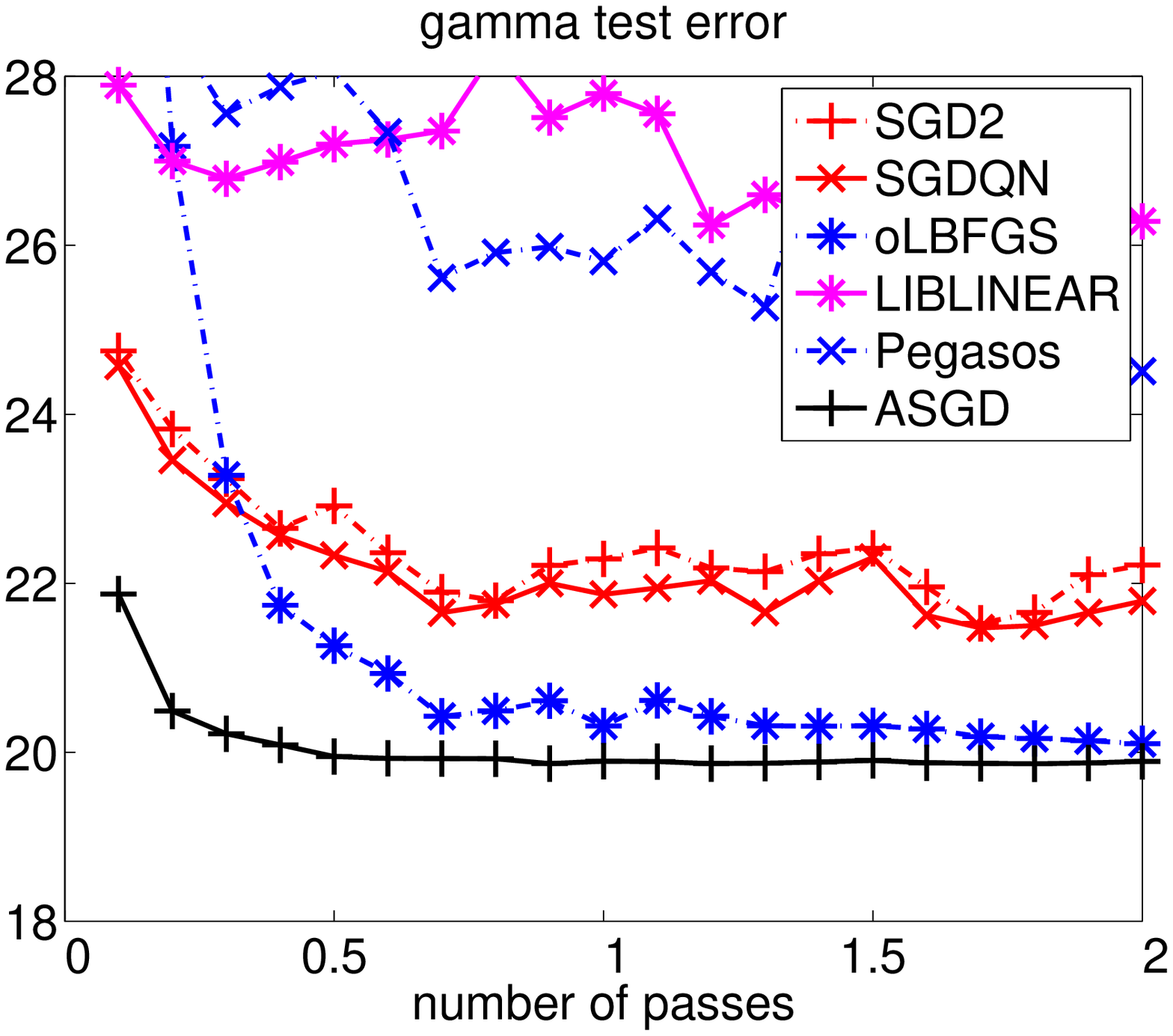} \\

\epsfxsize=4.7cm \epsfysize=4.5cm
\epsfbox{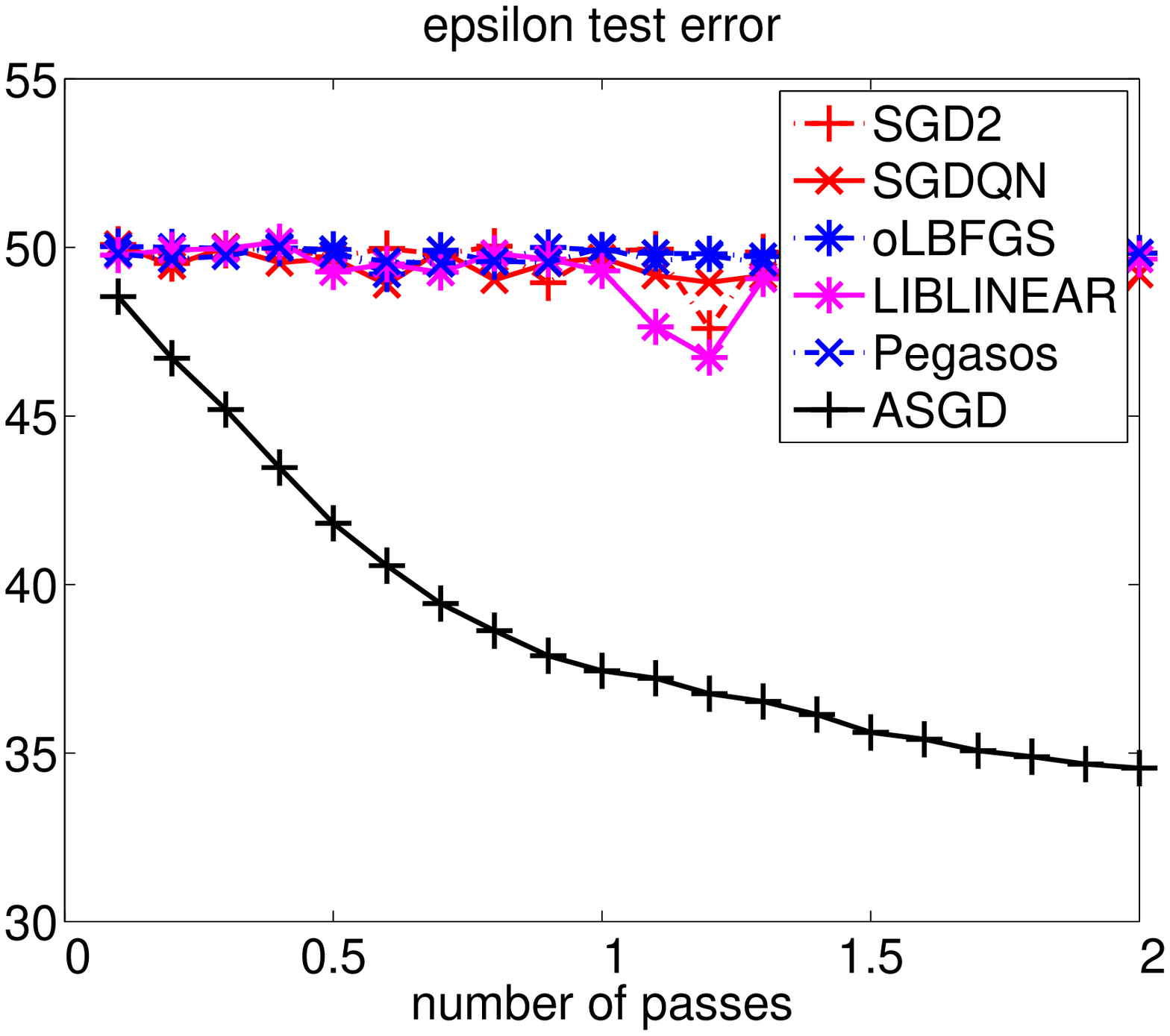} &
\epsfxsize=4.7cm \epsfysize=4.5cm
\epsfbox{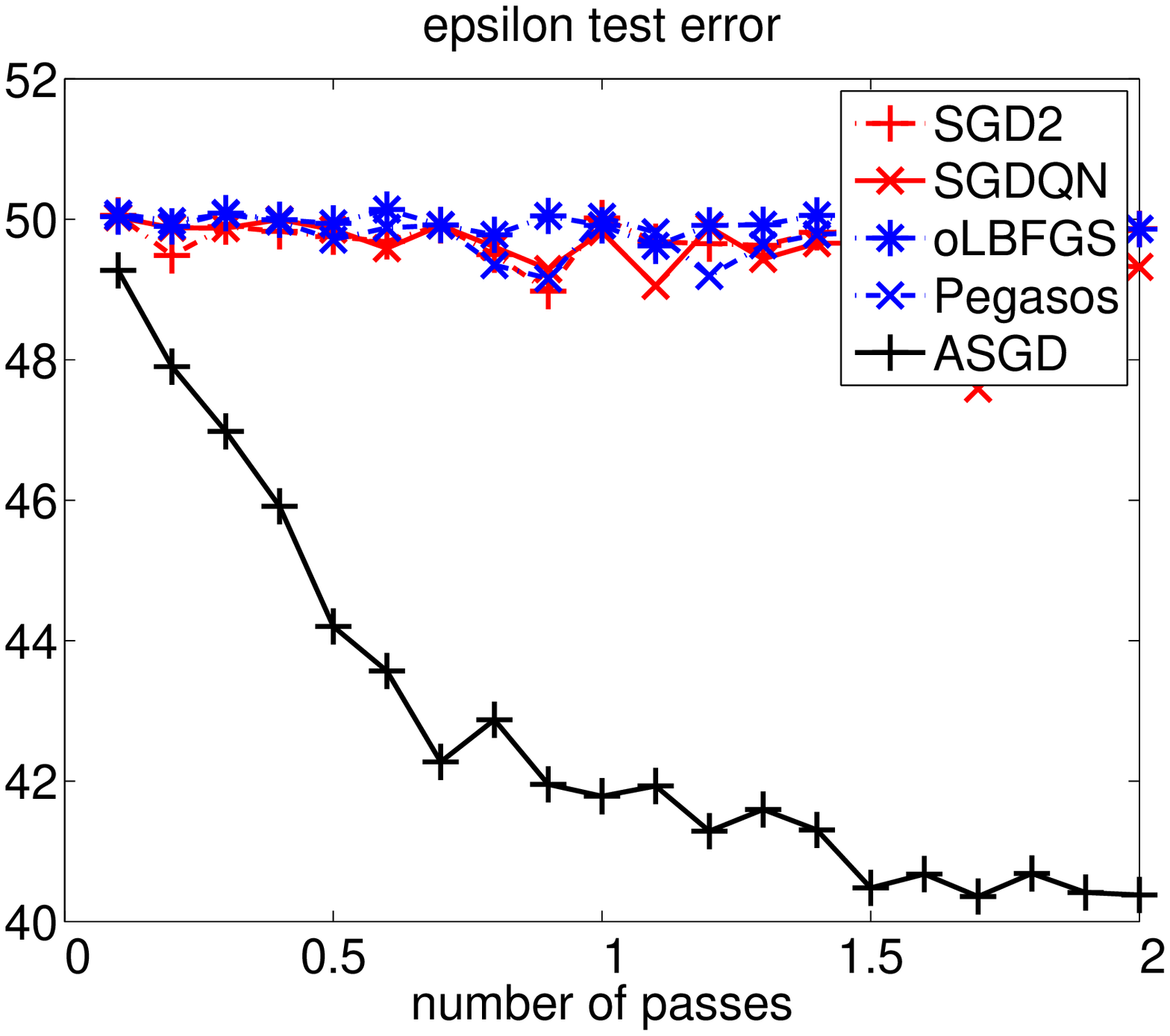} &
\epsfxsize=4.7cm \epsfysize=4.5cm
\epsfbox{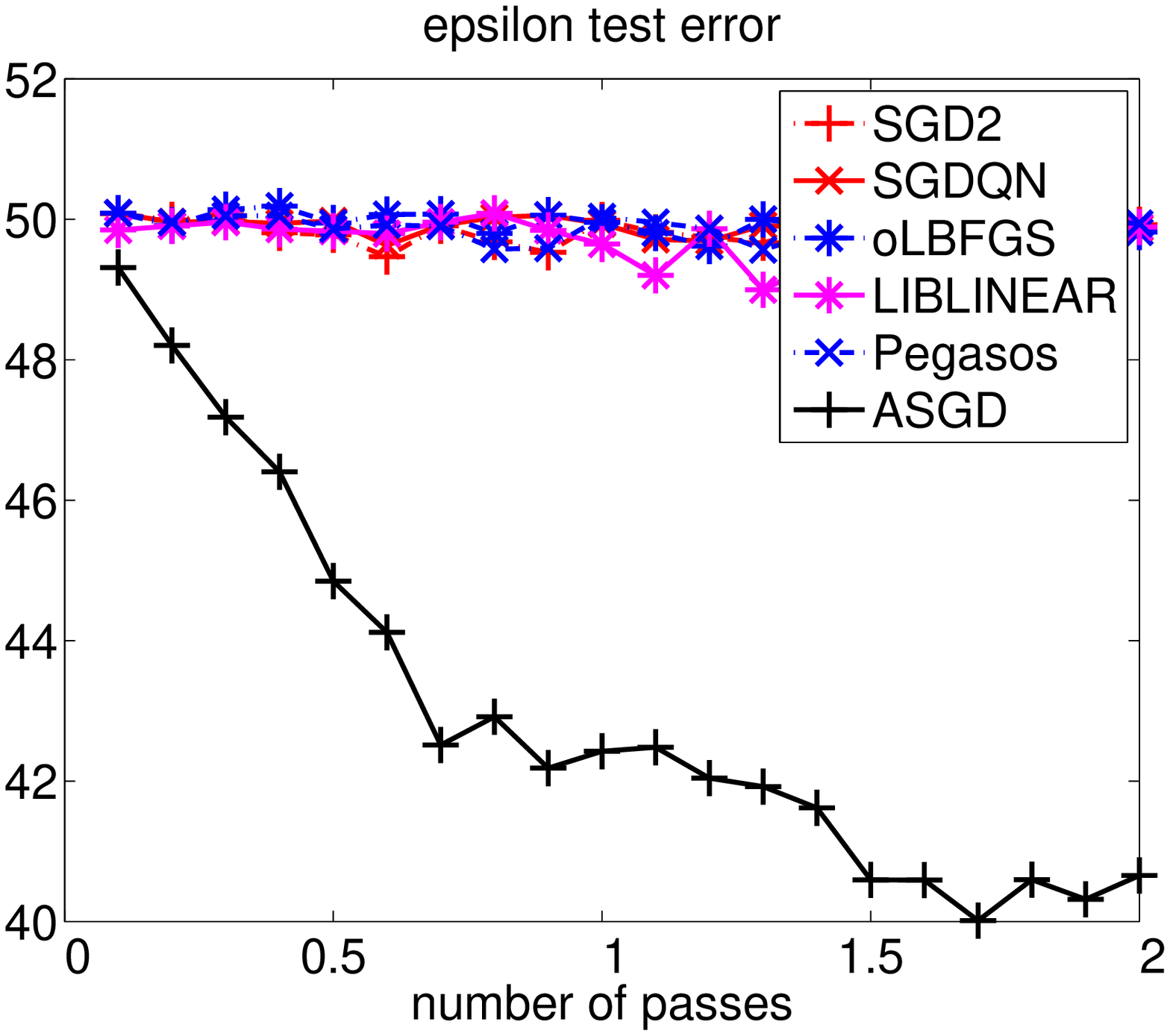} \\
\end{tabular}

\caption{\label{FigSet1} Test error (\%) vs. number of passes.
Left: L2SVM; Middle: logistic regression; Right: SVM.}
\end{figure}

\begin{figure}[ht]
\begin{tabular}{ccc}
\epsfxsize=4.7cm \epsfysize=4.5cm
\epsfbox{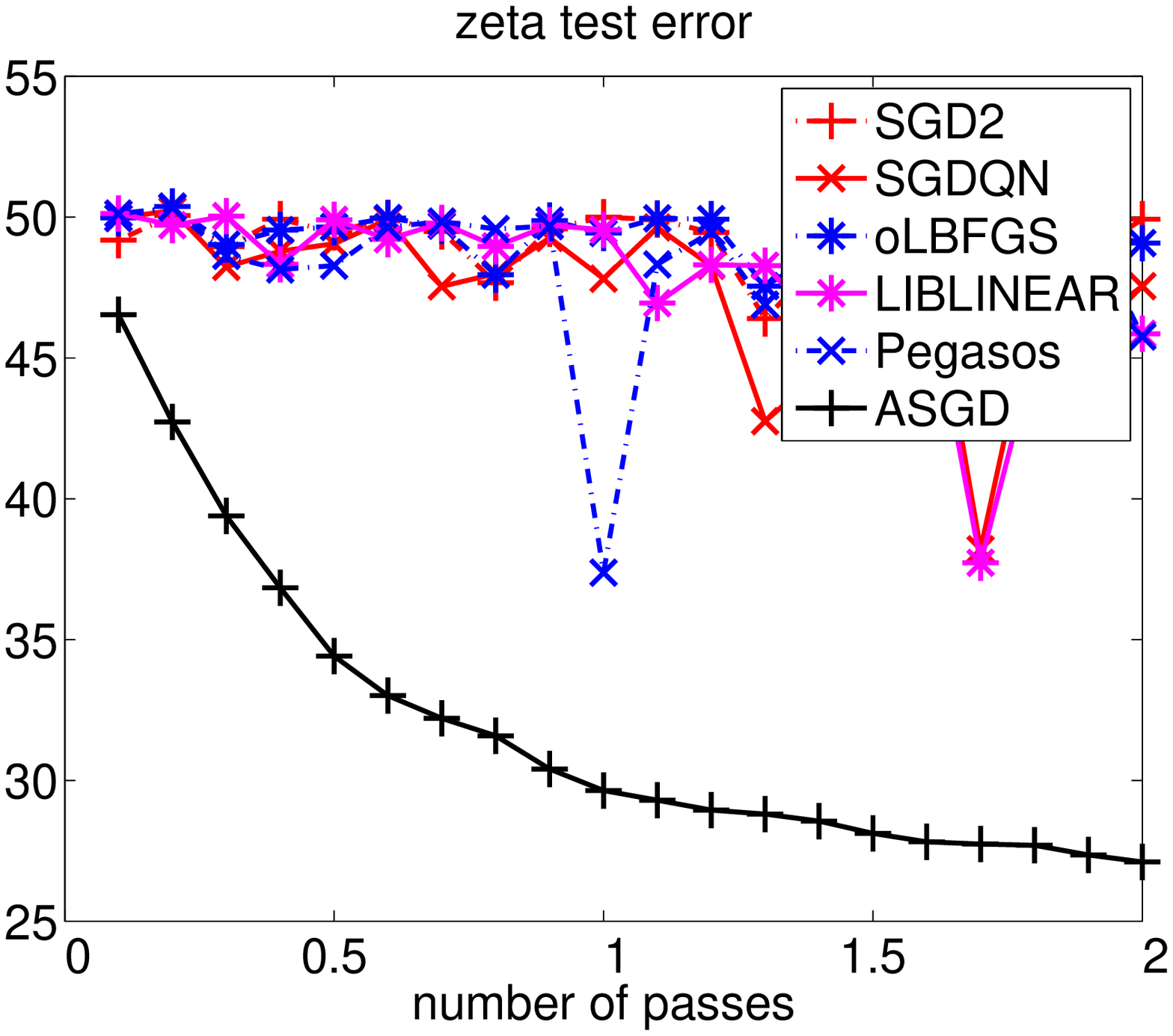} &
\epsfxsize=4.7cm \epsfysize=4.5cm
\epsfbox{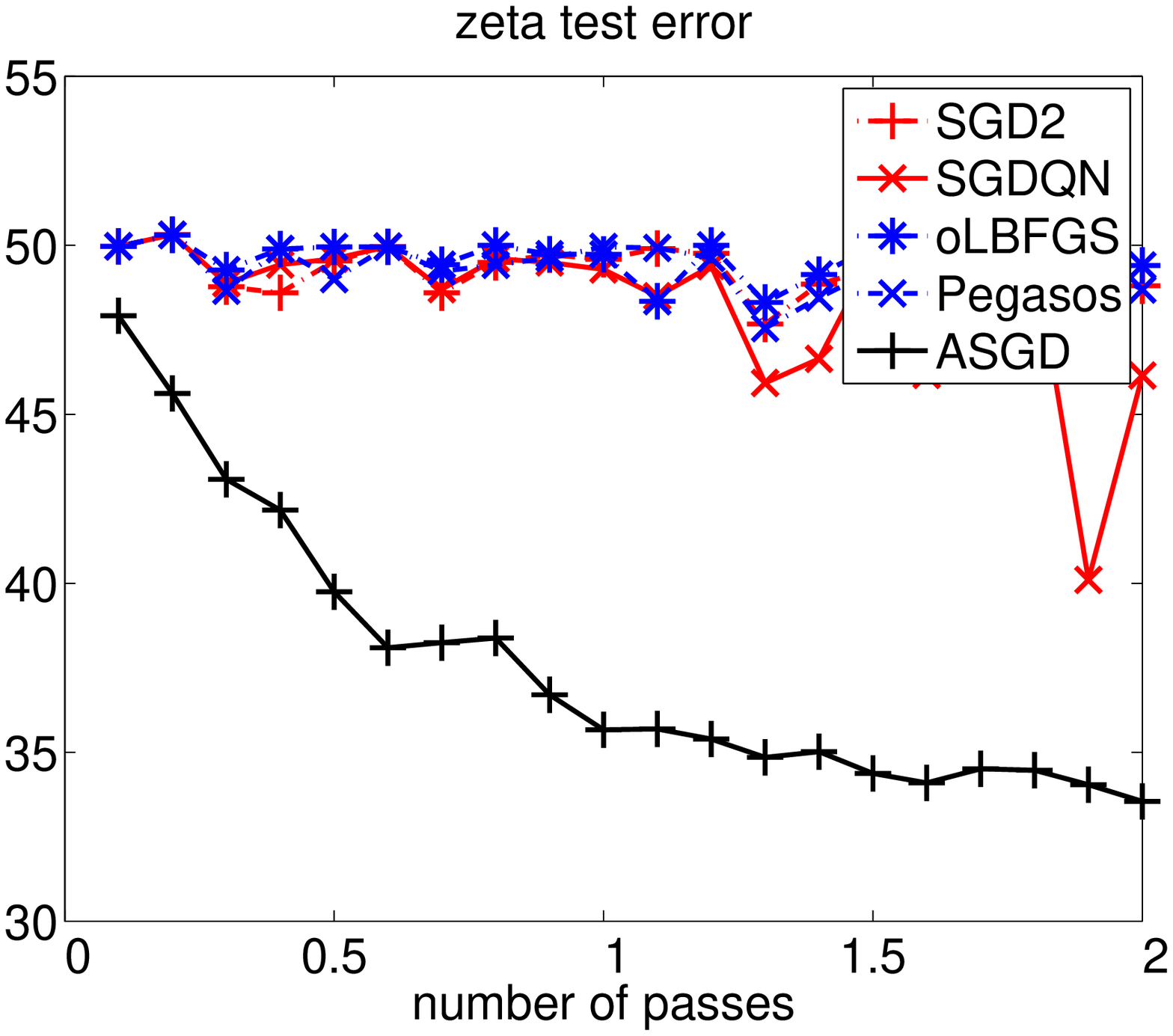} &
\epsfxsize=4.7cm \epsfysize=4.5cm
\epsfbox{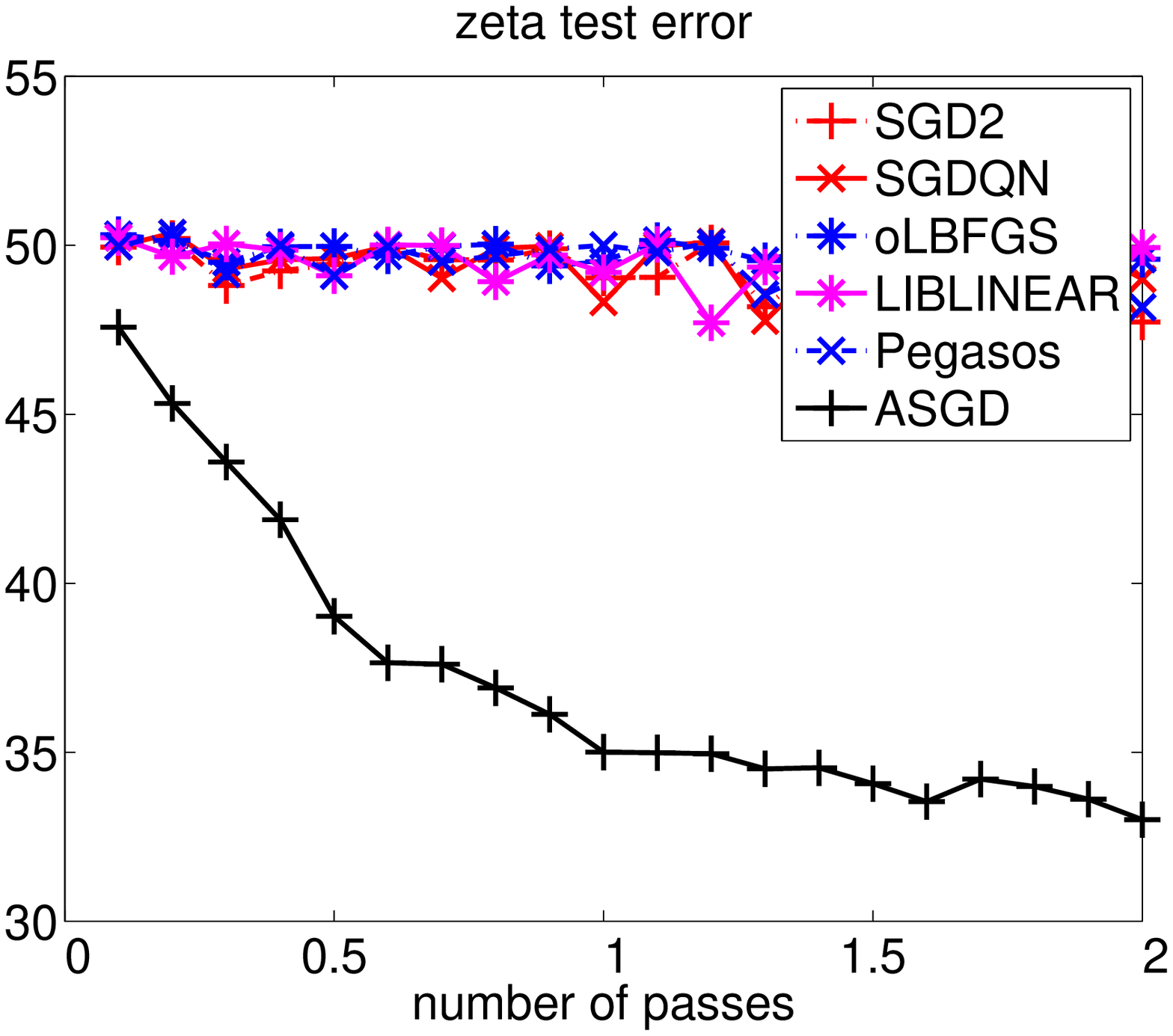} \\

\epsfxsize=4.7cm \epsfysize=4.5cm
\epsfbox{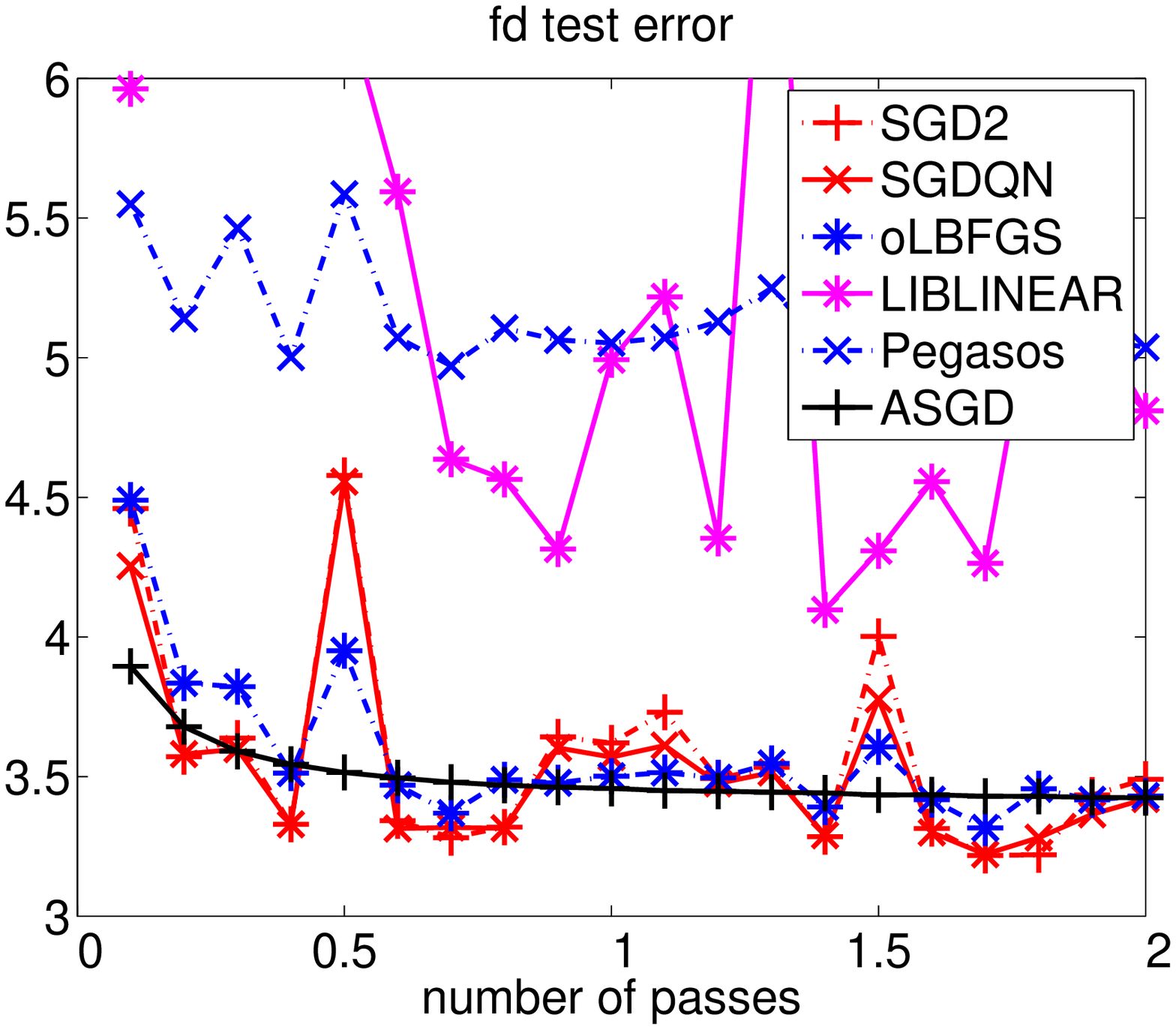} &
\epsfxsize=4.7cm \epsfysize=4.5cm
\epsfbox{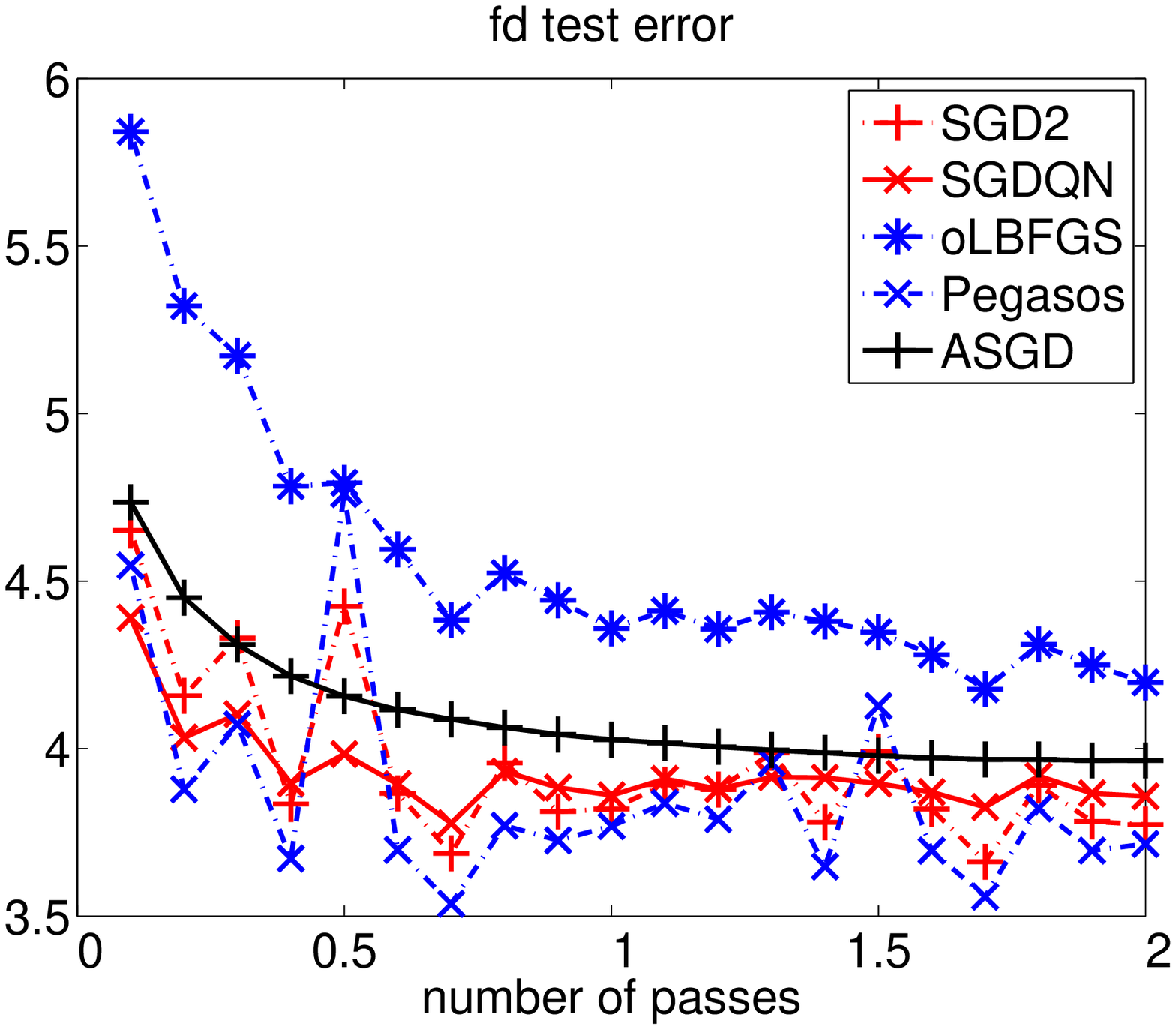} &
\epsfxsize=4.7cm \epsfysize=4.5cm
\epsfbox{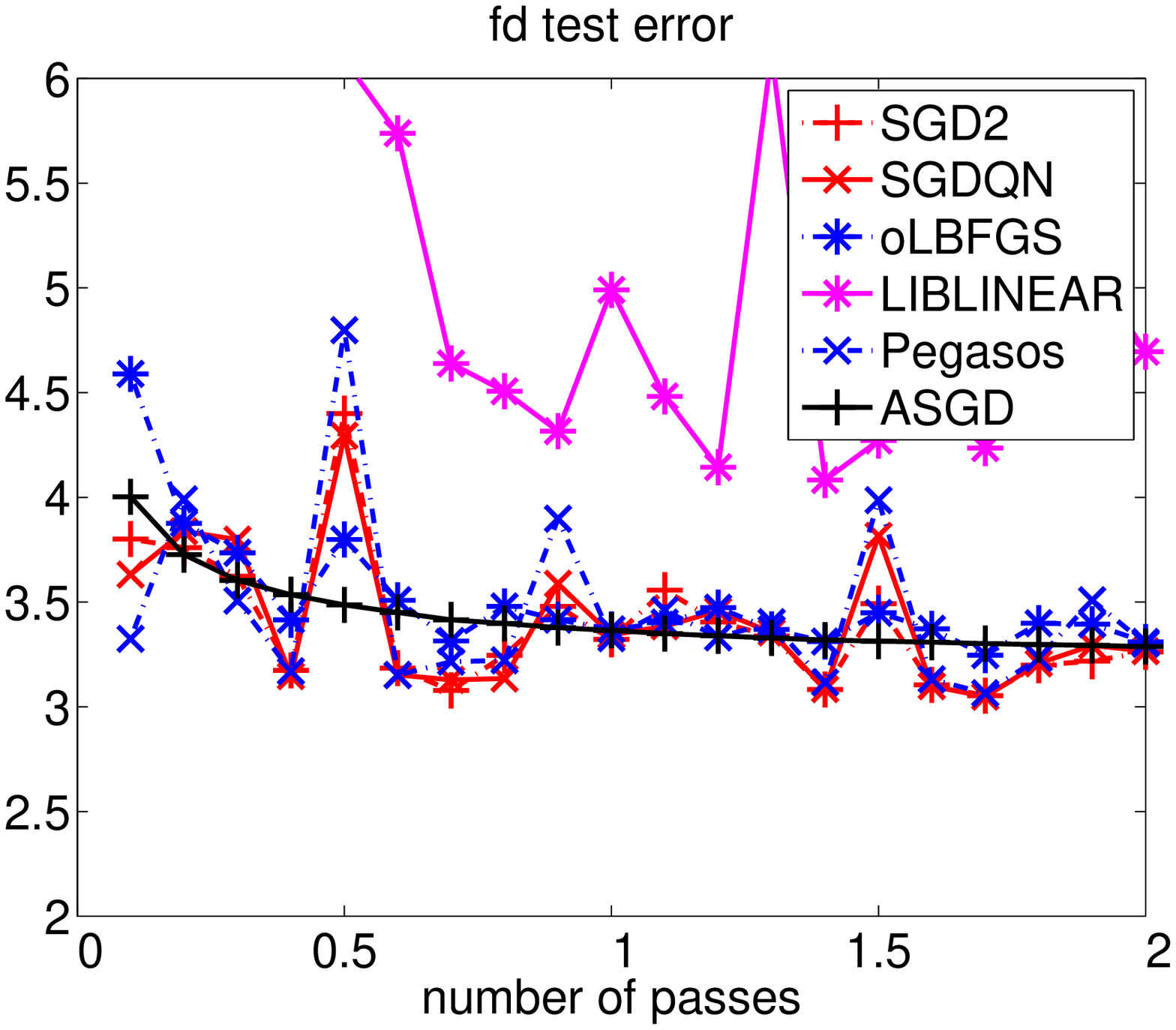} \\

\epsfxsize=4.7cm \epsfysize=4.5cm
\epsfbox{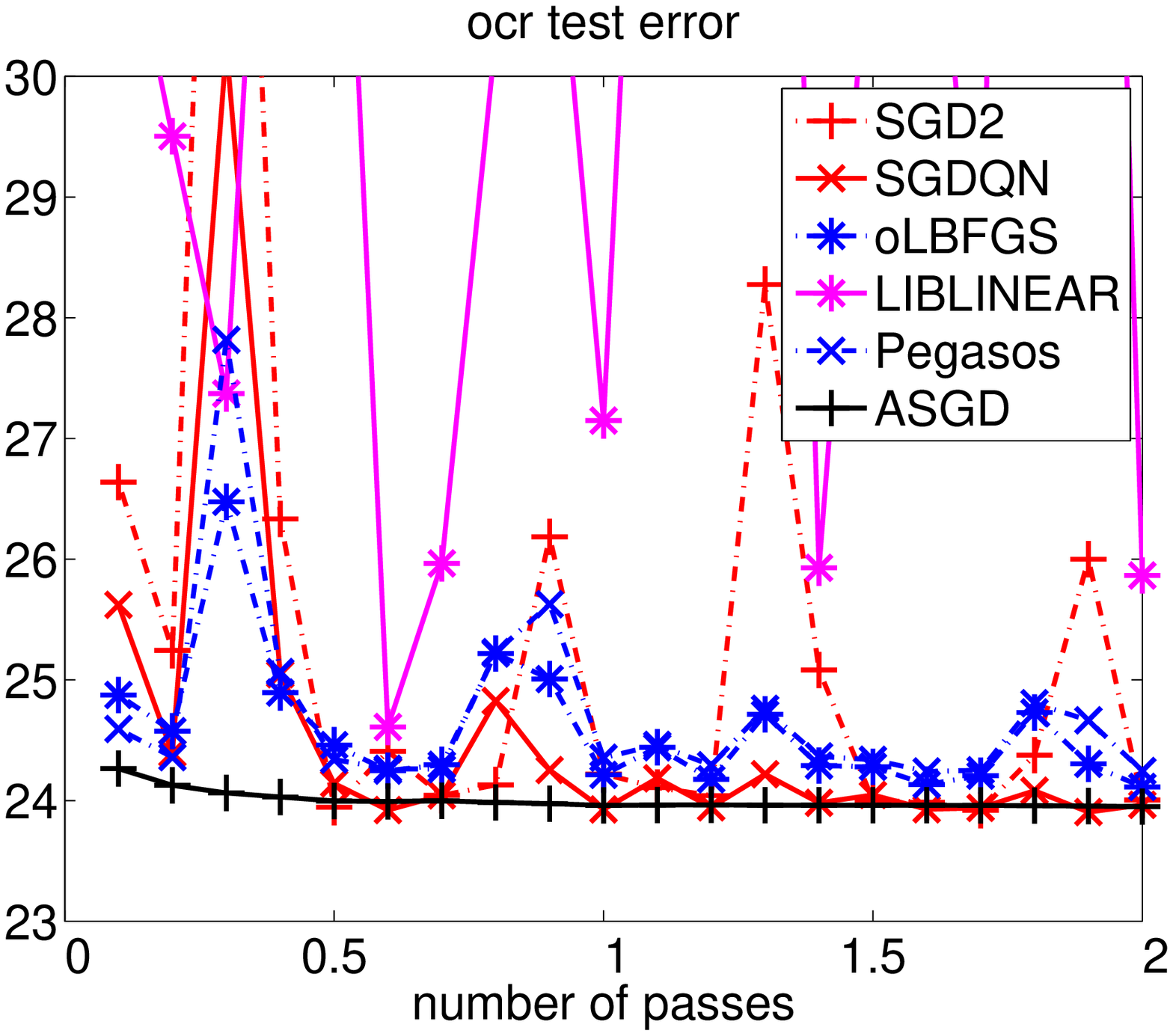} &
\epsfxsize=4.7cm \epsfysize=4.5cm
\epsfbox{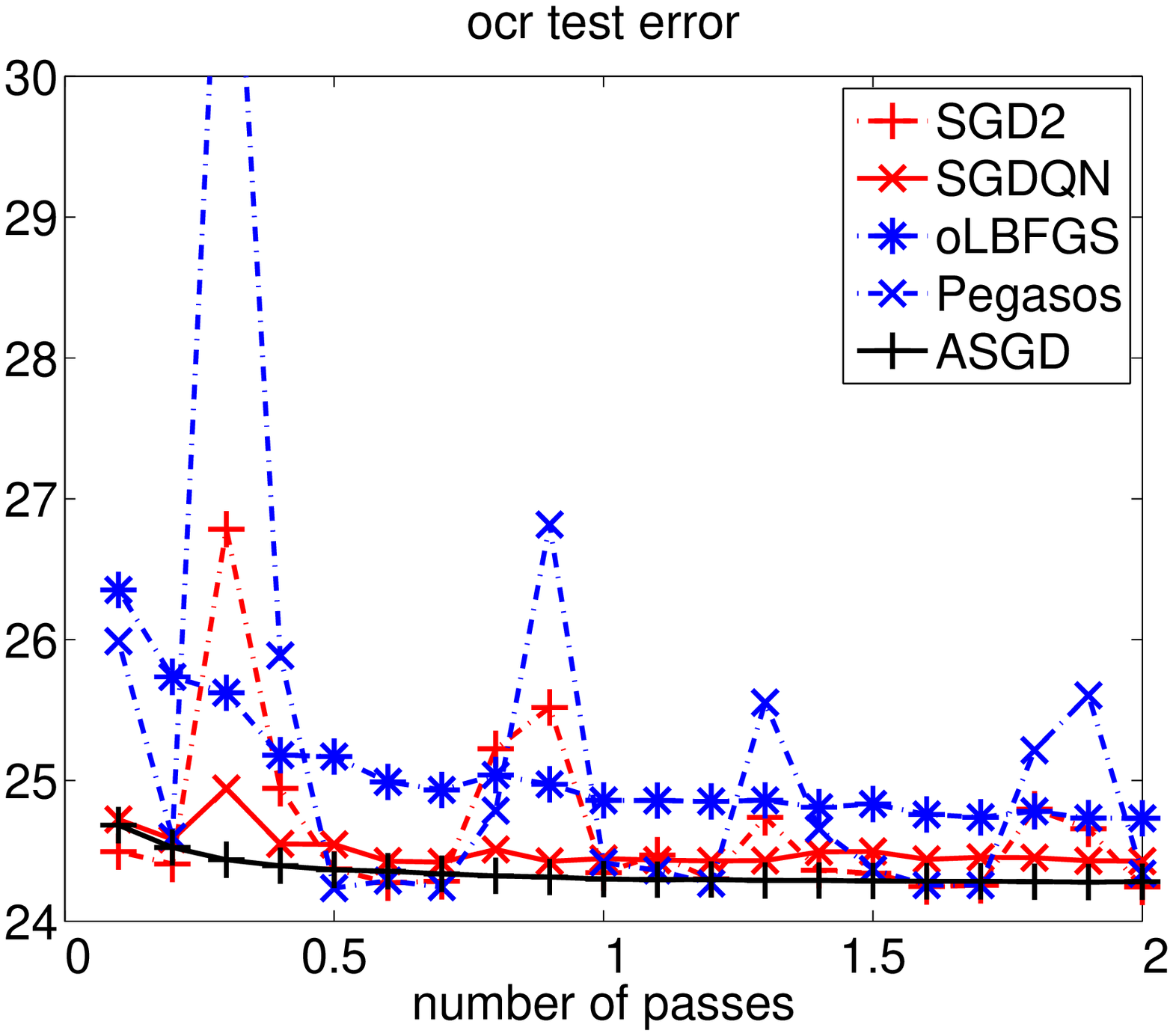} &
\epsfxsize=4.7cm \epsfysize=4.5cm
\epsfbox{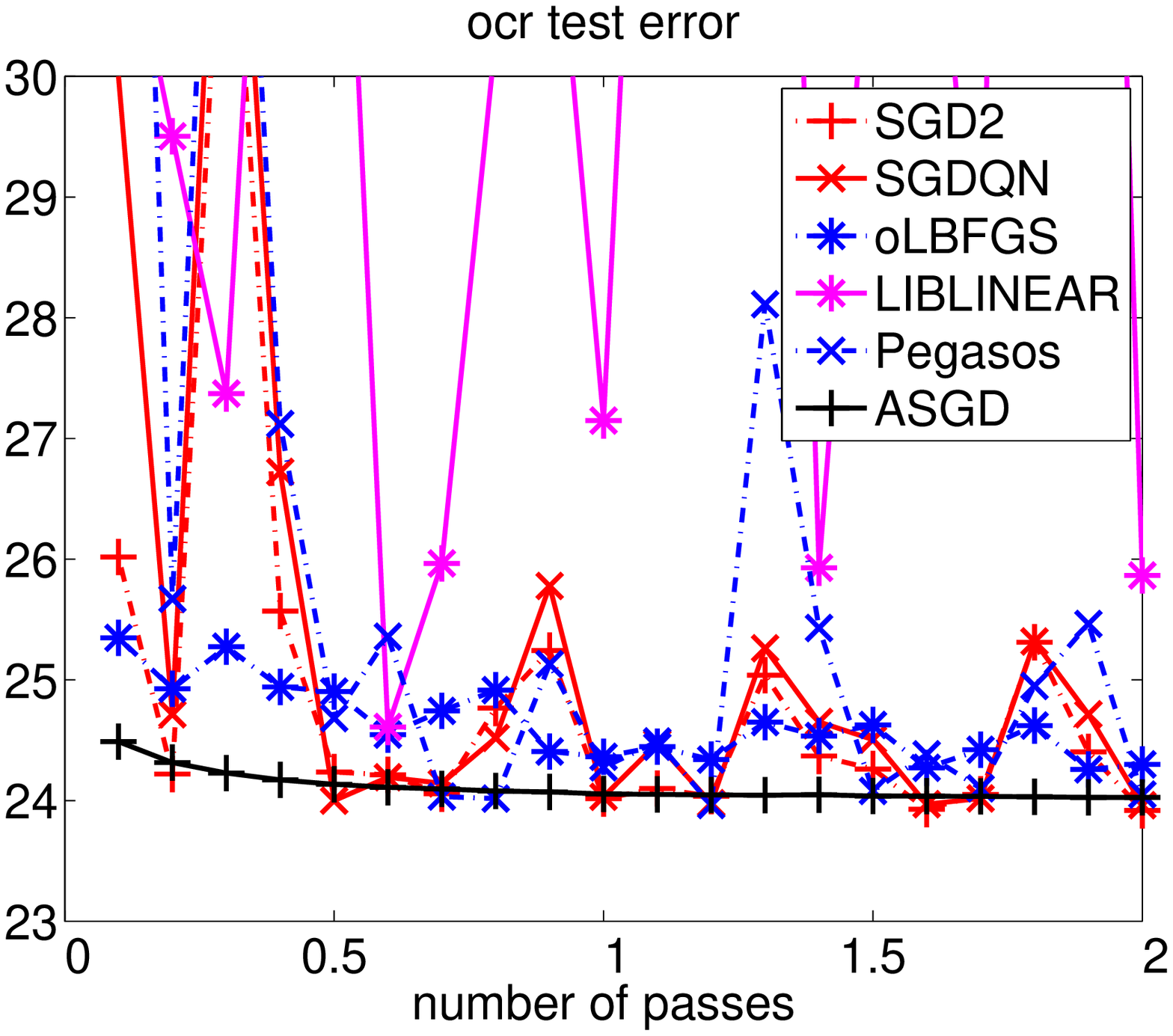} \\

\epsfxsize=4.7cm \epsfysize=4.5cm
\epsfbox{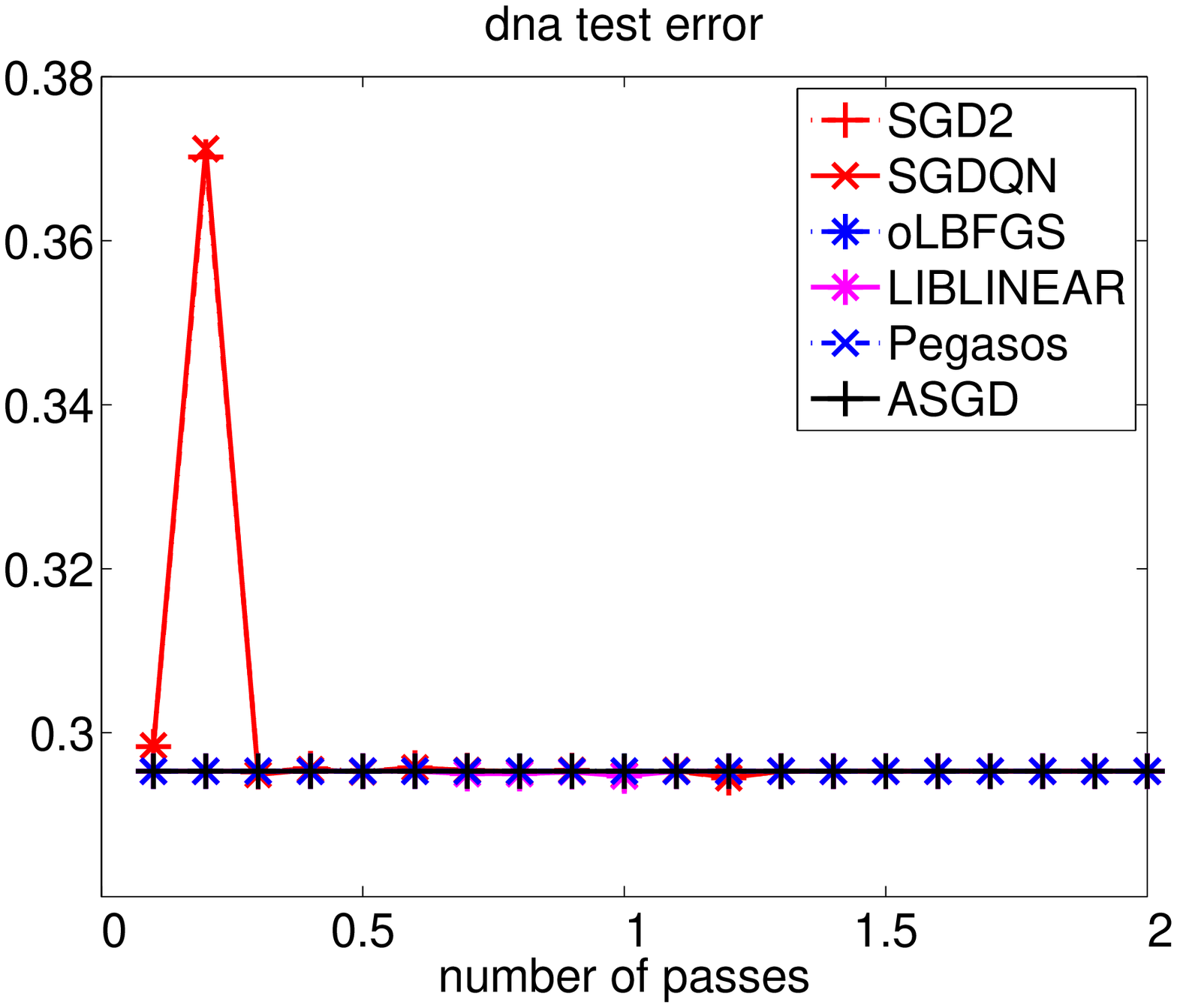} &
\epsfxsize=4.7cm \epsfysize=4.5cm
\epsfbox{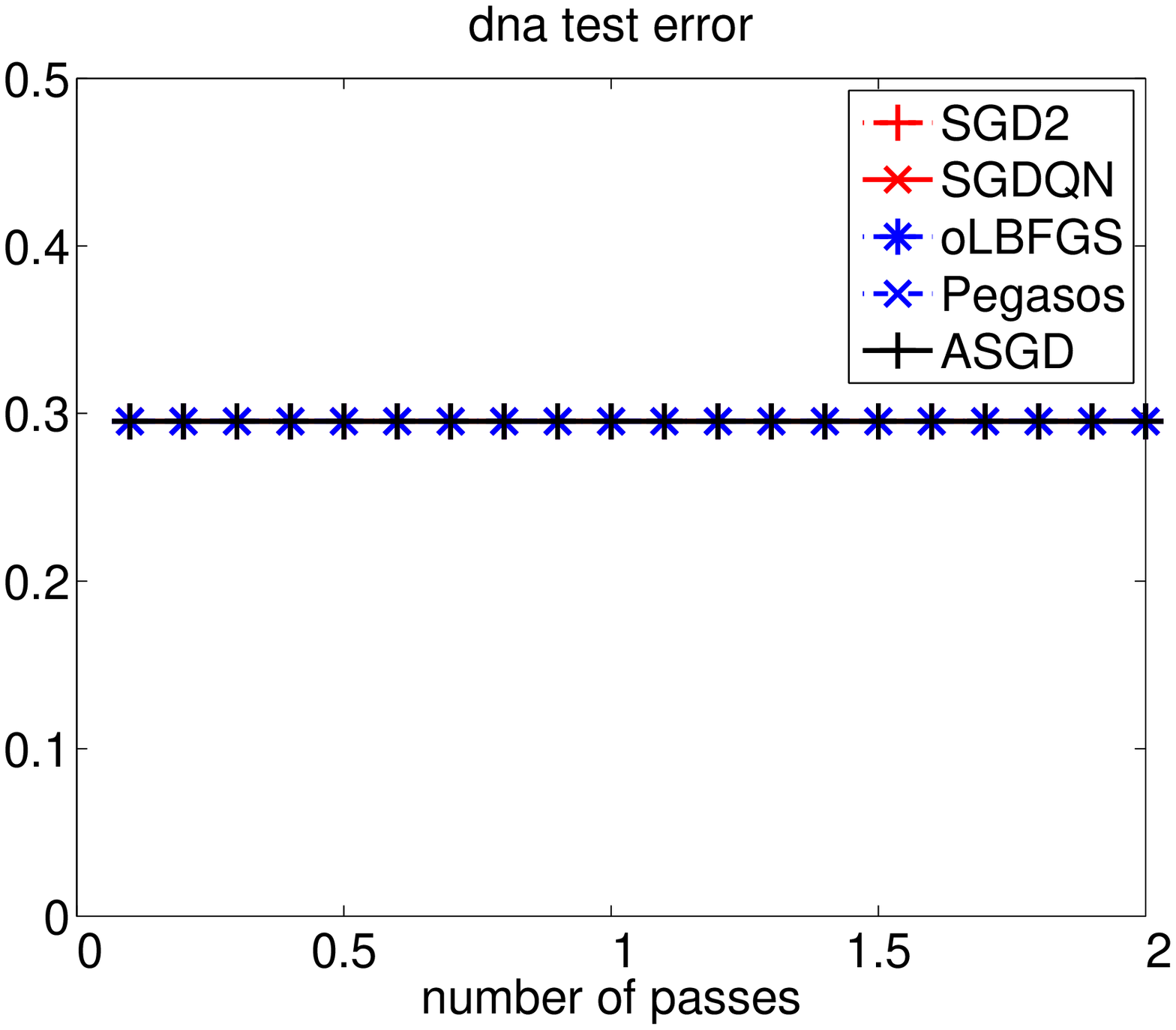} &
\epsfxsize=4.7cm \epsfysize=4.5cm
\epsfbox{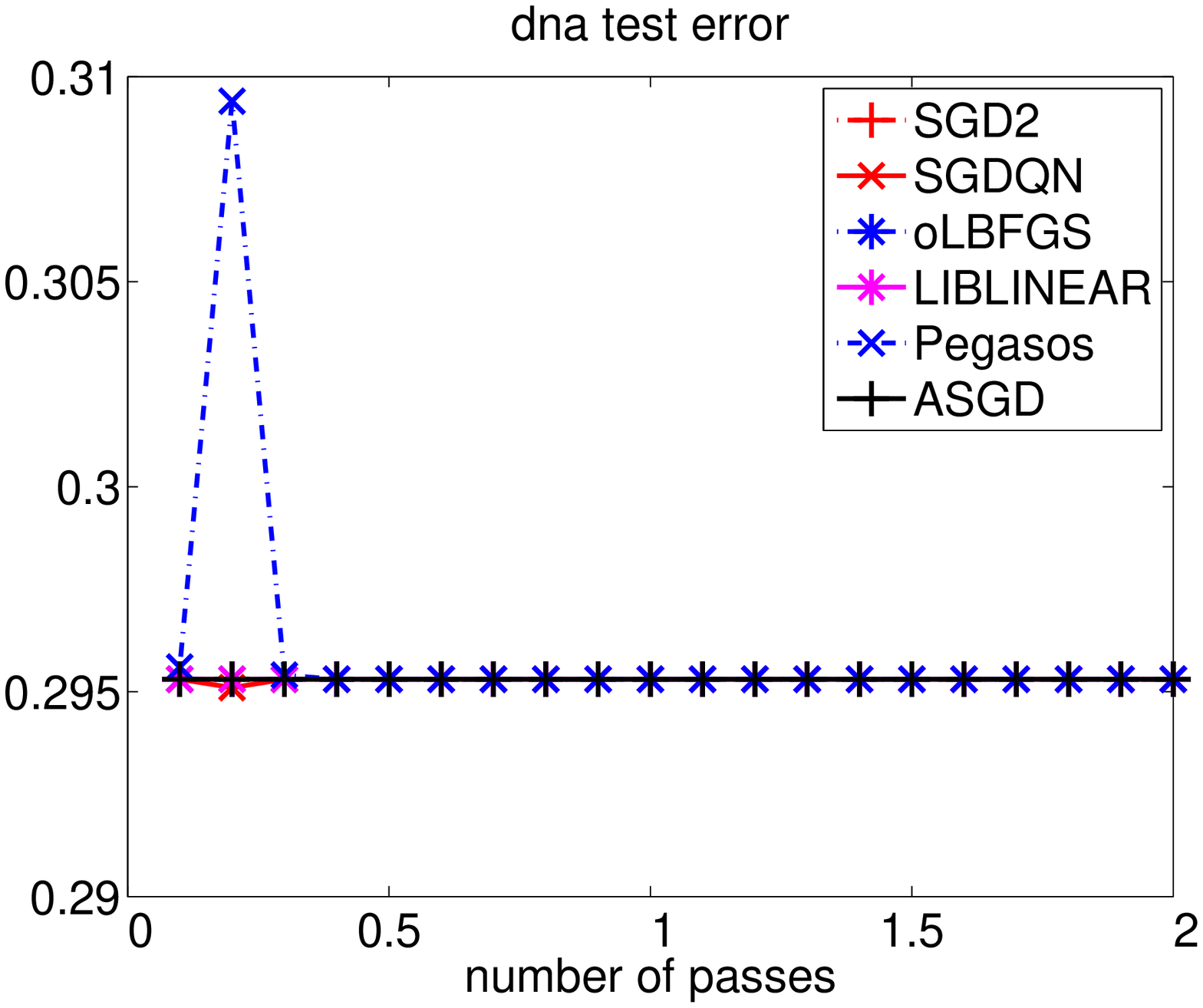} \\
\end{tabular}

\caption{\label{FigSet2} Test error (\%) vs. number of passes.
Left: L2SVM; Middle: logistic regression; Right: SVM.}
\end{figure}

\section{Conclusion} \label{SecConclusion}
ASGD is relatively easy to implement compared to other algorithms.
And as demonstrated on both synthetic and real data sets, with our
proposed learning rate schedule, ASGD performs better than other
more complicated algorithms for large scale learning problems. In
this paper, we only apply ASGD to linear models with convex loss,
which has unique local optimum. It would be more interesting to
see how ASGD can be applied to more complicated models such as
conditional random fields (CRF) or models with multiple local
optimums such as neural networks.

\acks{The author would like to thank Leon Bottou for the
insightful discussions, Antoine Bordes for providing source code
of SGDQN, SGD2 and oLBFGS, and Yi Zhang for the suggestions to
improve the exposition of this paper.}

\bibliography{asgd}

\begin{thebibliography}{20}
\providecommand{\natexlab}[1]{#1}
\providecommand{\url}[1]{\texttt{#1}}
\expandafter\ifx\csname urlstyle\endcsname\relax
  \providecommand{\doi}[1]{doi: #1}\else
  \providecommand{\doi}{doi: \begingroup \urlstyle{rm}\Url}\fi

\bibitem[Amari et~al.(2000)Amari, Park, and Fukumizu]{Amari00}
{Shun-ichi} Amari, Hyeyoung Park, and Kenji Fukumizu.
\newblock Adaptive method of realizing natural gradient learning for multilayer
  perceptrons.
\newblock \emph{Neural Computation}, 12:\penalty0 1399--1409, 2000.

\bibitem[Bordes et~al.(2009)Bordes, Bottou, and Gallinari]{Bordes09}
Antoine Bordes, L\'{e}on Bottou, and Patrick Gallinari.
\newblock {SGD-QN}: Careful quasi-{N}ewton stochastic gradient descent.
\newblock \emph{Journal of Machine Learning Research}, 10:\penalty0 1737--1754,
  2009.

\bibitem[Bottou(2007)]{Bottou07}
L\'{e}on Bottou.
\newblock Stochastic gradient descent on toy problems.
\newblock http://leon.bottou.org/projects/sgd, 2007.

\bibitem[Bottou and Bousquet(2008)]{Bottou08}
L\'{e}on Bottou and Olivier Bousquet.
\newblock The tradeoffs of large scale learning.
\newblock In J.C. Platt, D.~Koller, Y.~Singer, and S.~Roweis, editors,
  \emph{Advances in Neural Information Processing Systems 20}, pages 161--168.
  MIT Press, Cambridge, MA, 2008.

\bibitem[Bottou and {LeCun}(2005)]{Bottou05}
L\'{e}on Bottou and Yann {LeCun}.
\newblock On-line learning for very large datasets.
\newblock \emph{Apllied Stochastic Models in Business and Industry},
  21(21):\penalty0 137--151, 2005.

\bibitem[Fabian(1968)]{Fabian68}
V\'{a}clav Fabian.
\newblock On asymptotic normality in stochastic approximation.
\newblock \emph{The Annals of Mathematical Statistics}, 39\penalty0
  (4):\penalty0 1327--1332, 1968.

\bibitem[Fabian(1973)]{Fabian73}
V\'{A}clav Fabian.
\newblock Asymptotically efficient stochastic approximation; the {RM} case.
\newblock \emph{The Annals of Statistics}, 1\penalty0 (3):\penalty0 486--495,
  1973.

\bibitem[Fan et~al.(2008)Fan, Change, Hsieh, Wang, and Lin]{Fan08}
{Rong-En} Fan, {Kai-Wei} Change, {Cho-Jui} Hsieh, {Xiang-Rui} Wang, and
  {Chih-Jen} Lin.
\newblock {LIBLINEAR}: A library for large linear classification.
\newblock \emph{Journal of Machine Learning Research}, 9:\penalty0 1871--1874,
  2008.

\bibitem[Hazan et~al.(2006)Hazan, Kalai, and Agarwal]{Hazan06}
Elad Hazan, Adam Kalai, and Satyen Kale~Amit Agarwal.
\newblock Logarithmic regret algorithms for online convex optimization.
\newblock In \emph{Proceedings of the 19th Annual Conference on Learning
  Theory}, Pittsburgh, Pennsylvania, 2006.

\bibitem[Langford et~al.(2009)Langford, Li, and Zhang]{Langford09}
John Langford, Lihong Li, and Tong Zhang.
\newblock Sparse online learning via truncated gradient.
\newblock \emph{Journal of Machine Learning Research}, 10:\penalty0 777--801,
  2009.

\bibitem[LeCun et~al.(1998)LeCun, Bottou, Bengio, and Haffner]{LeCun98}
Yann LeCun, Leon Bottou, Yoshua Bengio, and Patrick Haffner.
\newblock Gradient-based learning applied to document recognition.
\newblock \emph{Proceedings of the IEEE}, 86(11):\penalty0 2278--2324, 1998.

\bibitem[Lewis et~al.(2004)Lewis, Yang, Rose, G., and Li]{Lewis04}
David~D. Lewis, Yiming Yang, Tony~G. Rose, G., and Fan Li.
\newblock {RCV1}: A new benchmark collection for text categorization research.
\newblock \emph{Journal of Machine Learning Research}, 5:\penalty0 361--397,
  2004.

\bibitem[Nemirovski et~al.(2009)Nemirovski, Juditski, Lan, and
  Shapiro]{Nemirovski09}
Arkadi Nemirovski, Anatoli Juditski, Guanghui Lan, and Alexander Shapiro.
\newblock Robust stochastic approximation approach to stochastic programming.
\newblock \emph{{SIAM} Journal on Control and Optimization}, 19(4):\penalty0
  1574--1609, 2009.

\bibitem[Polyak and Juditsky(1992)]{Polyak92}
Boris~T. Polyak and Anatoli.~B. Juditsky.
\newblock Acceleration of stochastic approximation by averaging.
\newblock \emph{Automation and Remote Control}, 30(4):\penalty0 838--855, 1992.

\bibitem[Roux et~al.(2008)Roux, Manzagol, and Bengio]{Roux08}
Nicolas~Le Roux, Pierre-Antoine Manzagol, and Yoshua Bengio.
\newblock Topmoumoute online natural gradient algorithm.
\newblock In J.C. Platt, D.~Koller, Y.~Singer, and S.~Roweis, editors,
  \emph{Advances in Neural Information Processing Systems 20}, pages 849--856.
  MIT Press, Cambridge, MA, 2008.

\bibitem[Schraudolph et~al.(2007)Schraudolph, Yu, and Günter]{Schraudolph07}
Nicol~N. Schraudolph, Jin Yu, and Simon Günter.
\newblock A stochastic quasi-newton method for online convex optimization.
\newblock In \emph{Proceedings of the 9$^{th}$ International Conference on
  Artificial Intelligence and Statistics (AISTAT)}, pages 433--440, 2007.

\bibitem[Shalev-Shwartz and Tewari(2009)]{Shalev-Shwartz09}
Shai Shalev-Shwartz and Ambuj Tewari.
\newblock Stochastic methods for $\ell_1$ regularized loss minimization.
\newblock In \emph{Proceedings of the 26$^{st}$ International Conference on
  Machine Learning (ICML)}, 2009.

\bibitem[Shalev-Shwartz et~al.(2007)Shalev-Shwartz, Shinger, and
  Srebro]{Shalev-Shwartz07}
Shai Shalev-Shwartz, Yoram Shinger, and Nathan Srebro.
\newblock Pegasos: {P}rimal {E}stimated sub-{G}r{A}dient {SO}lver for {SVM}.
\newblock In \emph{Proceedings of the 24$^{th}$ Fourth International Conference
  on Machine Learning (ICML)}, Corvallis, OR, 2007.

\bibitem[Sonnenburg et~al.(2008)Sonnenburg, Franc, {Yom-Tov}, and
  Sebag]{Sonnenburg08}
Soeren Sonnenburg, Vojtech Franc, Elad {Yom-Tov}, and Michele Sebag.
\newblock Pascal large scale learning challenge.
\newblock http://largescale.first.fraunhofer.de, 2008.

\bibitem[Zhang(2004)]{Zhang04}
Tong Zhang.
\newblock Solving large scale linear prediction problems using stochastic
  gradient descent algorithms.
\newblock In \emph{Proceedings of the 21$^{st}$ International Conference on
  Machine Learning (ICML)}, 2004.

\end{thebibliography}

\appendix
\section{Proofs}
\begin{lemma}\label{LemmaGamma} Let $\kappa=1-\max(0,2c-1)\frac{a}{\lambda_0}$. If $\gamma_0 \lambda_1 \le 1$,
then
\[ \left(\frac{1}{\gamma_{k+1}}-\frac{1}{\gamma_k}\right)\frac{1}{\gamma_{k+1}} \le \left(\frac{1}{\gamma_k}-\frac{1}{\gamma_{k-1}}\right)\frac{1}{\gamma_k} (1-\lambda_0\gamma_k)^{\kappa-1} \]
\end{lemma}
\begin{proof}
For $0<c\le 0.5$, let $f(x)=(x^c-(x-1)^c)x^c$, where
$x=k+\frac{1}{a\gamma_0}$. We only need to show $f'(x)\le 0$
\begin{eqnarray*}
f'(x) &=& 2cx^{2c-1}-c(x-1)^{c-1}x^c - c(x-1)^c x^{c-1} \\
&=&2cx^{c-1}(x^c-(x-1)^c-\frac{1}{2}(x-1)^{c-1}) \\
&\le&2cx^{c-1}((x-1)^c+c(x-1)^{c-1}-(x-1)^c-\frac{1}{2}(x-1)^{c-1}) \\
&=&c(2c-1)x^{c-1}(x-1)^{c-1} \le 0
\end{eqnarray*}
where we used the fact $x^c \le (x-1)^c+c(x-1)^{c-1}$ for $0\le c \le 1$.

For $c>0.5$, let $f(x)=\log((x^c-(x-1)^c)x^c)$, where
$x=k+\frac{1}{a\gamma_0}$. We only need to show
\[ f(x+1)-f(x)+\frac{a(2c-1)}{\lambda_0}\log(1- \lambda_0 \gamma_0 (a\gamma_0 x)^{-c}) \le 0 \]
By mean value theorem, there exists some $y: x\le y\le x+1$ s.t.
$f(x+1)-f(x)=f'(y)$. Hence
\begin{eqnarray*}
&& f(x+1)-f(x)+\frac{a(2c-1)}{\lambda_0}\log((1-\lambda_0 \gamma_0 (a\gamma_0 x)^{-c}) \\
&\le& f'(y)-a(2c-1) \gamma_0 (a\gamma_0 x)^{-c} \le f'(y)-(2c-1)(a\gamma_0)^{1-c} y^{-c} \\
&=& \frac{2c(y^c-(y-1)^c-\frac{1}{2}(y-1)^{c-1})}{y(y^c-(y-1)^c)} - \frac{(2c-1)(a\gamma_0 y)^{1-c}}{y} \\
&\le& \frac{2c(y^c-(y-1)^c-\frac{1}{2}(y-1)^{c-1})}{y(y^c-(y-1)^c)} - \frac{2c-1}{y} \\
&=& \frac{y^c-(y-1)^c-c(y-1)^{c-1}}{y(y^c-(y-1)^c)} \le 0
\end{eqnarray*}
\end{proof}
The following is a key lemma which is used several times in this
paper.
\begin{lemma} \label{Lemma1}
Let $X_j^t$ and $\bar{X}_j^t$ be
\begin{eqnarray*}
&& X_j^t=\prod_{i=j}^t(I-\gamma_i A) \mbox{\quad,\quad} X_j^t=I
\mbox{ for } j>t \mbox{\quad,\quad} \bar{X}_j^t=\sum_{i=j}^t
\gamma_j X_{j+1}^i
\end{eqnarray*}
If $\gamma_0 \lambda_1 \le 1$ and $(2c-1)a<\lambda_0$, then we
have the following bound for $\bar{X}_j^t$.
\begin{eqnarray*}
(I-X_j^t)A^{-1}  \le \bar{X}_j^t  \le (1+ c_0(1+a \gamma_0
j)^{c-1})A^{-1} \le (1+c_0)A^{-1}
\end{eqnarray*}
where $c_0$ is the same as in Theorem \ref{ThmLinearAvg}.
\end{lemma}
\begin{proof}
It is easy to verify the following relation by induction on $t$,
\begin{eqnarray}
&& \sum_{i=j}^t \gamma_i X_j^{i-1} = (I-X_j^t)A^{-1} \label{EqSumX}
\end{eqnarray}
Now we calculate the difference between $\bar{X}_j^t$ and $\sum_{i=j}^t \gamma_i X_j^{i-1}$.
\begin{eqnarray*}
&& \bar{X}_j^t - \sum_{i=j}^t \gamma_i X_j^{i-1} = \sum_{i=j}^t (\gamma_j-\gamma_i) X_{j+1}^{i-1} =\sum_{i=j}^t \frac{\gamma_j-\gamma_i}{\gamma_i}\gamma_i X_{j+1}^{i-1} \\
&=& \sum_{i=j}^t \sum_{k=j+1}^i \left(\frac{\gamma_j}{\gamma_k}-\frac{\gamma_j}{\gamma_{k-1}}\right)\gamma_i X_{j+1}^{i-1} = \sum_{k=j+1}^t \left(\frac{\gamma_j}{\gamma_k}-\frac{\gamma_j}{\gamma_{k-1}}\right) \sum_{i=k}^t  \gamma_i X_{j+1}^{i-1} \\
&=& \sum_{k=j+1}^t \left(\frac{\gamma_j}{\gamma_k}-\frac{\gamma_j}{\gamma_{k-1}}\right) \left(\sum_{i=j+1}^t  \gamma_i X_{j+1}^{i-1} - \sum_{i=j+1}^{k-1}  \gamma_i X_{j+1}^{i-1}\right)\\
&=& \sum_{k=j+1}^t \left(\frac{\gamma_j}{\gamma_k}-\frac{\gamma_j}{\gamma_{k-1}}\right) A^{-1}(I-X_{j+1}^t-I+X_{j+1}^{k-1}) \\
&=& -\left(\frac{\gamma_j}{\gamma_t}-1\right) A^{-1} X_{j+1}^t + \gamma_j A^{-1} \sum_{k=j+1}^t \left(\frac{1}{\gamma_k}-\frac{1}{\gamma_{k-1}}\right) X_{j+1}^{k-1}
\end{eqnarray*}
It is clear that from the first line of above equation that
$\bar{X}_j^t - \sum_{i=j}^t \gamma_i X_j^{i-1}>0$. Hence we obtain
the first inequality of the lemma. We have
\[ (1-\lambda_0 \gamma_k)^{-1} I \le (I-\gamma_k A)^{-1} \]
By Lemma \ref{LemmaGamma}, we have
\[ \left(\frac{1}{\gamma_{k+1}}-\frac{1}{\gamma_k}\right)\frac{1}{\gamma_{k+1}} I \le \left(\frac{1}{\gamma_k}-\frac{1}{\gamma_{k-1}}\right)\frac{1}{\gamma_k} (I-\gamma_k A)^{\kappa-1}\]
Hence
\[ \left(\frac{1}{\gamma_k}-\frac{1}{\gamma_{k-1}}\right)\frac{1}{\gamma_k} X_{j+1}^{k-1}\le \left(\frac{1}{\gamma_{j+1}}-\frac{1}{\gamma_j}\right)\frac{1}{\gamma_{j+1}} (X_{j+1}^{k-1})^{\kappa}\]
Define $Y_j^k$ as  $Y_j^k = \prod_{i=j}^k (I-\kappa \gamma_i A)$. Since $0<\kappa\le 1$, we have $(X_j^k)^\kappa \le Y_j^k$. Hence
\begin{eqnarray*}
\bar{X}_j^t - \sum_{i=j}^t \gamma_i X_j^{i-1} &\le& -\left(\frac{\gamma_j}{\gamma_t}-1\right) A^{-1} X_{j+1}^t + \gamma_j \left(\frac{1}{\gamma_{j+1}}-\frac{1}{\gamma_j}\right)\frac{1}{\gamma_{j+1}} A^{-1}\sum_{k=j+1}^t \gamma_k (X_{j+1}^{k-1})^{\kappa} \\
&\le& - \left(\frac{\gamma_j}{\gamma_t}-1\right) A^{-1} X_{j+1}^t + \frac{\gamma_j-\gamma_{j+1}}{\gamma_{j+1}^2} A^{-1}\sum_{k=j+1}^t \gamma_k Y_{j+1}^{k-1} \\
&=& - \left(\frac{\gamma_j}{\gamma_t}-1\right) A^{-1} X_{j+1}^t + \frac{\gamma_j-\gamma_{j+1}}{\kappa \gamma_{j+1}^2} A^{-2} (I - Y_{j+1}^t )\\
&\le& \frac{\gamma_0}{\kappa \gamma_1} \frac{\gamma_j-\gamma_{j+1}}{ \gamma_j \gamma_{j+1}} A^{-2} = \frac{1}{\kappa \gamma_1}( (1+a \gamma_0 (j+1))^c -(1+a \gamma_0 j))^c ) A^{-2} \\
&\le&  \frac{ac\gamma_0(1+a \gamma_0 j)^{c-1}}{\kappa \gamma_1}  A^{-2} \le  \frac{ac\gamma_0(1+a \gamma_0 j)^{c-1}}{\kappa \gamma_1} \frac{ A^{-1}}{\lambda_0} \\
&=& c_0(1+a \gamma_0 j)^{c-1} A^{-1}
\end{eqnarray*}
Now plugging (\ref{EqSumX}) into above inequality, we obtain the claim of the lemma.
\end{proof}
With Lemma \ref{Lemma1}, we can now prove Theorem \ref{ThmLinearAvg}.
\begin{proof}(Theorem \ref{ThmLinearAvg})
From (\ref{EqSGDLinear}), we get
\begin{eqnarray}
&& \Delta_t=\Delta_{t-1}-\gamma_t (A\Delta_{t-1}+\xi_t) \mbox{\quad,\quad} \bar{\Delta}_t=\frac{1}{t}\sum_{i=1}^t \Delta_i \label{EqDelta}
\end{eqnarray}
From (\ref{EqDelta}), we have
\[ \Delta_t=\prod_{j=1}^t(I-\gamma_jA)\Delta_0 + \sum_{j=1}^t \prod_{i=j+1}^t(I-\gamma_i A)\gamma_j \xi_j \]
then
\begin{eqnarray*}
 \bar{\Delta}_t &=& \frac{1}{t} \sum_{j=1}^t \Delta_j =  \frac{1}{t} \sum_{j=1}^{t}  \prod_{i=1}^j (I-\gamma_i A)\Delta_0 + \frac{1}{t} \sum_{j=1}^t  \left(\sum_{k=j}^t \prod_{i=j+1}^k(I-\gamma_iA)\right) \gamma_j \xi_j \\
 &=&  \frac{1}{\gamma_0 t} (\bar{X}_0^t-\gamma_0 I) \Delta_0 + \frac{1}{t} \sum_{j=1}^t \bar{X}_j^t \xi_j = I^{(0)} + I^{(1)}
  \end{eqnarray*}
where $\bar{X}_j^t$ is defined in Lemma \ref{Lemma1}. Hence
\begin{eqnarray}
&& t E(\|I^{(0)}\|_A^2)= \frac{1}{\gamma_0^2t} \Delta_0^T A(\bar{X}_0^t-\gamma_0 I)^2 \Delta_0 \le \frac{(1 + c_0)^2 }{\gamma_0^2 t} \Delta_0^T  A^{-1} \Delta_0 \label{EqI0Bound}
\end{eqnarray}
\begin{eqnarray}
&& t E(\|I^{(1)}\|_A^2)= \frac{1}{t}\sum_{j=1}^t E(\xi_j^T A(\bar{X}_j^t)^2 \xi_j )\le \frac{1}{t} \sum_{j=1}^t (1+c_0(1+a \gamma_0 j)^{c-1})^2 E(\xi_t^T A^{-1}\xi_t) \nonumber \\
&\le& \left(1+\frac{2c_0+c_0^2}{t} \sum_{j=1}^t (1+a \gamma_0 j)^{c-1} \right) \tr(A^{-1} S)\le \left(1+\frac{(2c_0+c_0^2)((1+a \gamma_0 t)^c-1)}{a c \gamma_0 t} \right)  \tr(A^{-1} S) \nonumber \\
&\le& \left(1+\frac{(2c_0+c_0^2)(1+a \gamma_0 t)^{c-1}}{c} \right)  \tr(A^{-1} S) \label{EqI1Bound}
\end{eqnarray}
And we have $E((I^{(0)})^T A I^{(1)})=0$ since $E(\xi_j)=0$.
\end{proof}

\begin{proof}(Lemma \ref{LemmaRegression})
\begin{eqnarray}
&& t E\|I^{(2)}\|_A^2 =  t E \left\| \frac{1}{t}\sum_{j=1}^t \bar{X}_j^t \xi_j^{(2)} \right\|_A^2 =  \frac{1}{t} \sum_{j=1}^t E  \|\bar{X}_j^t \xi_j^{(2)}\|_A^2 \nonumber \\
&=& \frac{1}{t} \sum_{j=1}^t E( \xi_j^{(2)T} A (\bar{X}_j^t)^2  \xi_j^{(2)}) \le \frac{1}{t} \sum_{j=1}^t (1+c_0)^2 E( \xi_j^{(2)T} A^{-1} \xi_j^{(2)}) \nonumber \\
&\le& \frac{1}{t} \sum_{j=1}^t (1+c_0)^2 c_1 E(\|\Delta_{j-1}\|_A^2) \le  \frac{(1+c_0)^2 c_1 }{t}\left((1+c_2)\|\Delta_0\|_A^2 + c_3 \sum_{j=1}^{t-1} \gamma_j \right)\nonumber \\
&\le& \frac{(1+c_0)^2 c_1}{t} \left((1+c_2)\|\Delta_0\|_A^2 + \frac{c_3((1+a\gamma_0 t)^{1-c}-1)}{a(1-c)}\right) \nonumber \\
&\le& (1+c_0)^2 c_1 \left(\frac{1+c_2}{t} \|\Delta_0\|_A^2 + \frac{c_3 \gamma_0}{1-c}(1+a\gamma_0t)^{-c}  \right) \nonumber
\end{eqnarray}
\end{proof}

\begin{proof}(Lemma \ref{LemmaC1})
Let $\Sigma_x=E(xx^T)$. We have the following:
\begin{eqnarray}
&& g(\theta,d)=\pian{l(\theta,d)}{\theta}=xx^T \theta - xy \nonumber \\
&& \bar{g}(\theta)=E(g(\theta,d)) = \Sigma_x \theta - E(xy) \nonumber \\
&& A=\Sigma_x \mbox{\quad,\quad} b=E(xy) \mbox{\quad,\quad} \theta^*=A^{-1}b \nonumber \\
&& \xi^{(2)} = g(\theta,d)-g(\theta^*,d)-\bar{g}(\theta) = (xx^T-\Sigma_x)(\theta-\theta^*) \nonumber \\
&& E \left(\left.\| \xi^{(2)} \|_{A^{-1}}^2 \right| \theta \right)
= (\theta-\theta^*)^T E(x x^T A^{-1} x x ^T - \Sigma_x A^{-1}
\Sigma_x)(\theta-\theta^*) \label{EqXi2}
\end{eqnarray}
By the assumption of this lemma, we get
\begin{equation} E(x x^T A^{-1} x x^T)\le \frac{1}{\lambda_0}E(x x^T x x^T) \le \frac{M}{\lambda_0} A \label{EqyyyyBound} \end{equation}
From (\ref{EqXi2}) and (\ref{EqyyyyBound}), we get
\[ E \left(\left.\| \xi^{(2)} \|_{A^{-1}}^2 \right| \theta \right) \le \frac{M}{\lambda_0}\|\theta-\theta^*\|_A^2 \]
\end{proof}
\begin{lemma} \label{LemmaMaxGamma}
For linear regression problem $l(\theta,x,y)=\frac{1}{2}(\theta^T
x-y)^2$, assuming all $\|x\|^2$ are $M$, then (\ref{EqSGD}) will
diverge if learning rate is greater than $\frac{2}{M}$.
\end{lemma}
\begin{proof}
Let $X_i^t$ be defined as in Lemma \ref{Lemma1}. We obtain the
following from (\ref{EqSGD}),
\begin{eqnarray*}
\Delta_t=(I-\gamma_t x_tx_t^T)\Delta_{t-1} - \gamma_t(x_tx_t^T
\theta^*-x_t y_t)
\end{eqnarray*}
Let $A_t=x_tx_t^T$, $b_t=x_ty_t$, $A=E(A_t)$, $b=E(b_t)$. Taking
expectation with respect to $x_t,y_t$, noticing that
$A\theta^*=b$, we get
\begin{eqnarray*}
E(\Delta_t|\theta_{t-1})&=&(I-\gamma_t A)\Delta_{t-1} \\
E(\|\Delta_t\|^2|\Delta_{t-1}) &=& \Delta_{t-1}^T E(I  -2 \gamma_t A+\gamma_t^2 A_t A_t) \Delta_{t-1}  \\
&&+ \gamma_t^2E(\|A_t \theta^*-b_t\|^2)   + 2 \gamma_t^2 E(\theta^{*T} A_tA_t- b_t^T A_t)\Delta_{t-1} \\
&=& \|\Delta_{t-1}\|^2  -(2 \gamma_t - M\gamma_t^2)
\|\Delta_{t-1}\|_A^2 + \gamma_t^2 \tr(S)  + 2 \gamma_t^2 u^T
\Delta_{t-1}
\end{eqnarray*}
where $S=E((A_t \theta^*-b_t)(A_t\theta^*-b_t)^T)$, $u=E(A_tA_t
\theta^* - A_t b_t)$. Hence
\begin{eqnarray*}
&& E(\|\Delta_t\|^2) = E(\|\Delta_{t-1}\|^2)  -(2 \gamma_t - M\gamma_t^2) E(\|\Delta_{t-1}\|_A^2) + \gamma_t^2 \tr(S)  + 2 \gamma_t^2 u^T X_1^{t-1}\Delta_0
\end{eqnarray*}
If $\gamma_t>=\frac{2}{M}+\delta>\frac{2}{M}$, then
\begin{eqnarray*}
 E(\|\Delta_t\|^2)\ge E(\|\Delta_{t-1}\|^2) + \delta(2 +\delta M) E(\|\Delta_{t-1}\|_A^2) + \gamma_t^2 \tr(S)  + 2 \gamma_t^2 u^T X_1^{t-1}\Delta_0 \\
 \ge (1+\lambda_0 \delta(2 +\delta M)) E(\|\Delta_{t-1}\|^2) + \gamma_t^2 \tr(S)  + 2 \gamma_t^2 u^T X_1^{t-1}\Delta_0
\end{eqnarray*}
Noticing that $X_1^{t-1}\rightarrow 0$ as $t \rightarrow \infty$, we conclude that  $E(\|\Delta_t\|^2)$ is diverging if $\gamma_t\ge\frac{2}{M}$.
\end{proof}

\begin{proof}(Lemma \ref{LemmaNonQuadratic})
Let $\gamma_i^t=\sum_{j=i}^t \gamma_j$,
\begin{eqnarray*}
t E\|I^{(3)}\|_A^2 &\le& \frac{1}{t}\sum_{j=1}^t E\|\bar{X}_j^t \xi_j^{(3)}\|_A^2 + \frac{2}{t}\sum_{j=1}^t \sum_{k=j+1}^t E(\xi_j^{(3)T} \bar{X}_j^t A \bar{X}_k^t \xi_k^{(3)}) \\
&\le&  \frac{1}{t}\sum_{j=1}^t (1+c_0)^2 E\|\xi_j^{(3)}\|_{A^{-1}}^2 + \frac{2}{t}\sum_{j=1}^t \sum_{k=j+1}^t (1+c_0)^2 E(\|\xi_j^{(3)}\|_{A^{-1}} \|\xi_k^{(3)}\|_{A^{-1}}) \\
&\le&  \frac{(1+c_0)^2c_4^2}{t}\left(\sum_{j=1}^t  E\|\Delta_j\|_A^4 + 2\sum_{j=1}^t \sum_{k=j+1}^t E(\|\Delta_j\|_A^2 \|\Delta_k\|_A^2) \right)\\
&\le&  \frac{(1+c_0)^2c_4^2}{t}\left(\sum_{j=1}^t  E\|\Delta_j\|_A^4 + 2\sum_{j=1}^t E\left(\|\Delta_j\|_A^2 \sum_{k=j+1}^t E(\|\Delta_k\|_A^2 | \theta_j) \right) \right)\\
&\le&  \frac{(1+c_0)^2c_4^2}{t}\left(\sum_{j=1}^t  E\|\Delta_j\|_A^4 + 2\sum_{j=1}^t E\left(\|\Delta_j\|_A^2 \left(c_2\|\Delta_j\|_A^2 + c_3 \sum_{k=j+1}^t \gamma_k \right) \right) \right)\\
&\le&  \frac{(1+c_0)^2c_4^2}{t}\left((1+2c_2)\sum_{j=1}^t  E\|\Delta_j\|_A^4  + c_6 \gamma_1^t) +  2c_3  \sum_{j=1}^t E(\|\Delta_j\|_A^2) \sum_{k=j+1}^t \gamma_k \right)\\
&=&  \frac{(1+c_0)^2c_4^2}{t}\left((1+2c_2)(c_5 \|\Delta_0\|_A^4 + c_6 \gamma_1^t) +  2c_3  \sum_{k=2}^t \gamma_k  \sum_{j=1}^{k-1} E(\|\Delta_j\|_A^2) \right)\\
&\le&  \frac{(1+c_0)^2c_4^2}{t}\left((1+2c_2)(c_5 \|\Delta_0\|_A^4 + c_6 \gamma_1^t) +  2c_3  \sum_{k=2}^t \gamma_k (c_2 \|\Delta_0\|_A^2 + c_3\gamma_1^{k-1} )  \right)\\
&\le&  \frac{(1+c_0)^2c_4^2}{t}\left((1+2c_2)(c_5 \|\Delta_0\|_A^4 + c_6 \gamma_1^t) +  2c_2  c_3 \|\Delta_0\|_A^2 \gamma_1^t + c_3^2 (\gamma_1^t)^2 \right)\\
&\le&  \frac{(1+c_0)^2c_4^2}{t}\left((1+2c_2) c_5 \|\Delta_0\|_A^4
+(2c_2  c_3 \|\Delta_0\|_A^2 + (1+2c_2)c_6)\gamma_1^t + c_3^2
(\gamma_1^t)^2 \right)
\end{eqnarray*}
\end{proof}

\end{document}